\documentclass[11pt]{article} 

\usepackage{natbib}
\usepackage{fullpage}

\usepackage{amsmath,amsfonts,amssymb,amsthm}
\usepackage{bm}

\usepackage{hyperref}
\hypersetup{
    colorlinks,
    linkcolor=blue, 
    citecolor=blue, 
    linktoc=all 
  }

\usepackage{subcaption}

\renewcommand{\leq}{\leqslant}
\renewcommand{\geq}{\geqslant}

\newcommand{\di}{\mathrm{d}}

\newcommand{\eps}{\varepsilon}

\newcommand{\argmin}{\mathop{\mathrm{argmin}}}

\newcommand{\wt}{\widetilde}
\newcommand{\wh}{\widehat}

\newcommand{\N}{\mathbf N}

\newcommand{\R}{\mathbf R}

\usepackage{xspace}

\newcommand{\ie}{\textit{i.e.}\@\xspace} 

\newcommand{\iid}{i.i.d.\@\xspace}

\newcommand{\lint}{[\![}
\newcommand{\rint}{]\!]}
\renewcommand{\iint}[2]{\lint #1 , #2 \rint}

\newcommand{\E}{\mathbb E}
\renewcommand{\P}{\mathbb P}
\newcommand{\Qbb}{\mathbb Q}

\newcommand{\cond}{\,|\,}

\newcommand{\probas}{\mathcal{P}}
\newcommand{\indic}[1]{\bm 1 ( #1 )}
\newcommand{\kl}{\mathrm{KL}}
\newcommand{\kll}[2]{\kl ({#1}, {#2})}

\newcommand{\bernoullidist}{\mathcal{B}}

\newcommand{\betadist}{\mathsf{Beta}}

\newcommand{\uniformdist}{\mathcal{U}}

\newcommand{\F}{\mathcal{F}}
\newcommand{\X}{\mathcal{X}}

\newcommand{\A}{\mathcal{A}}

\newcommand{\pseudoregret}{\mathcal{R}}

\usepackage{graphicx}

\newtheorem{proposition}{Proposition}
\newtheorem{theorem}{Theorem}
\newtheorem{lemma}{Lemma}
\newtheorem{corollary}{Corollary}

\theoremstyle{definition}
\newtheorem{definition}{Definition}

\theoremstyle{remark}
\newtheorem{remark}{Remark}

\newtheorem{example}{Example}

\title{On the optimality of the Hedge algorithm in the stochastic regime}

\author{
  Jaouad Mourtada\footnote{Centre de Math\'ematiques Appliqu\'ees, \'Ecole Polytechnique, Palaiseau, France} 
  \qquad St\'ephane Ga\"iffas\footnote{LPSM, UMR 8001, Universit\'e Paris Diderot, Paris, France}
}

\begin{document}



\maketitle

\begin{abstract}
  In this paper, we study the behavior of the Hedge algorithm in the online stochastic setting.
  We prove that anytime Hedge with decreasing learning rate, which is one of the simplest algorithm for the problem of prediction with expert advice, is {remarkably}  
  both worst-case optimal and adaptive to the easier stochastic and adversarial with a gap problems.
  This shows that, in spite of its small, non-adaptive learning rate, Hedge possesses the same optimal regret guarantee in the stochastic case as recently introduced adaptive algorithms.
  Moreover, our analysis exhibits qualitative differences with other versions of the Hedge algorithm, such as the fixed-horizon variant (with constant learning rate) and the one based on the so-called ``doubling trick'', both of which fail to adapt to the easier stochastic setting.
  Finally, we determine the intrinsic limitations of anytime Hedge in the stochastic case, and discuss the improvements provided by
  more adaptive algorithms.

  \medskip

\end{abstract}

\section{Introduction}
\label{sec:introduction}

The standard setting of \emph{prediction with expert advice} \citep{littlestone1994weighted,freund1997boosting,vovk1998mixability,cesabianchi2006plg} aims to provide sound strategies for sequential prediction that combine the forecasts from different sources.
More precisely, in the so-called \emph{Hedge problem} \citep{freund1997boosting}, at each round the learner has to output a probability distribution on a finite set of \emph{experts} $\{1, \ldots, M\}$; the losses of the experts are then revealed, and the learner incurs the expected loss from its chosen probability distribution. 
The goal is then to control the \emph{regret}, defined as the difference between the cumulative loss of the learner and that of the best expert (with smallest loss).
This online prediction problem is typically considered in the \emph{individual sequences} framework, where the losses may be arbitrary and in fact set by an adversary that seeks to maximize the regret.
This leads to regret bounds that hold under virtually no assumption~\citep{cesabianchi2006plg}.

In this setting, arguably the simplest and most standard strategy is the \emph{Hedge algorithm} \citep{freund1997boosting}, also called the \emph{exponentially weighted averaged forecaster} \citep{cesabianchi2006plg}.
This algorithm depends on a time-varying parameter $\eta_t$ called the \emph{learning rate}, which quantifies by how much the algorithm departs from its initial probability distribution to put more weight on the currently leading experts.
Given a known finite time horizon $T$, the standard tuning of the learning rate is fixed and given by $\eta_t = \eta \propto \sqrt{\log(M) / T}$, which guarantees an optimal worst-case regret of order $O( \sqrt{T \log M} )$.
Alternatively, when $T$ is unknown, one can set $\eta_t \propto \sqrt{\log (M) / t}$ at round $t$, which leads to an \emph{anytime} $O(\sqrt{T \log M})$ regret bound valid for all $T \geq 1$.

While worst-case regret bounds are robust and always valid, they turn out to be overly pessimistic in some situations.
A recent line of research \citep{cesabianchi2007secondorder,derooij2014followtheleader,gaillard2014secondorder,koolen2014learning,sani2014exploiting,koolen2015squint,luo2015adanormalhedge} designs algorithms that combine $O(\sqrt{T \log M})$ worst-case regret guarantees with an improved regret on easier instances of the problem.
An interesting example of such an easier instance is the
stochastic problem, where it is assumed that the losses are stochastic and that at each round the expected loss of a ``best'' expert is smaller than those of the other experts by some gap $\Delta$.
Such algorithms rely either on a more careful, data-dependent tuning of the learning rate $\eta_t$ \citep{cesabianchi2007secondorder,derooij2014followtheleader,koolen2014learning,gaillard2014secondorder}, or on more sophisticated strategies \citep{koolen2015squint,luo2015adanormalhedge}.
As shown by~\citet{gaillard2014secondorder} (see also \citealt{koolen2016combining}), one particular type of adaptive regret bounds (so-called \emph{second-order bounds}) 
implies at the same time a $O (\sqrt{T\log M})$ worst-case bound and a better \emph{constant} $O(\log (M) / \Delta)$ bound in the stochastic problem with gap $\Delta$.
Arguably starting with the early work on {second-order bounds} \citep{cesabianchi2007secondorder}, the design of online learning algorithms that combine robust worst-case guarantees with improved performance on easier instances has been an active research goal in recent years \citep{derooij2014followtheleader,gaillard2014secondorder,koolen2014learning,sani2014exploiting}.
However, to the best of our knowledge, existing work on the Hedge problem has focused on developing new adaptive algorithms rather than on analyzing the behavior of ``conservative'' algorithms in favorable scenarios. 
Owing to the fact that the standard Hedge algorithm is designed for --- and analyzed in --- the adversarial setting \citep{littlestone1994weighted,freund1997boosting,cesabianchi2006plg}, and that its parameters are not tuned adaptively to obtain better bounds in easier instances, it may be considered as overly conservative and not adapted to stochastic environments.

\paragraph{Our contribution.} 

This paper fills a gap in the existing literature by providing an analysis of the standard Hedge algorithm in the stochastic setting.
We show that the anytime Hedge algorithm with default learning rate $\eta_t \propto \sqrt{\log (M) / t}$  actually \emph{adapts} to the stochastic setting, in which it achieves an optimal \emph{constant} $O(\log (M) / \Delta)$ regret bound \emph{without any dedicated tuning} for the easier instance, which might be surprising at first sight.
This contrasts with previous works, which require the construction of new adaptive (and more involved) algorithms.
Remarkably, this property is \emph{not} shared by the variant of Hedge for a known fixed-horizon $T$ with constant learning rate $\eta \propto \sqrt{\log(M) / T}$, since it suffers a $\Theta (\sqrt{T \log M})$ regret even in easier instances.
This exhibits a strong difference between the performances of the anytime and the fixed-horizon variants of the Hedge algorithm.

Given the aforementioned adaptivity of Decreasing Hedge, one may wonder whether there is in fact any benefit in using more sophisticated algorithms in the stochastic regime.
  We answer this question affirmatively, by considering a more refined measure of complexity of a stochastic instance than the gap $\Delta$.
  Specifically, we show that Decreasing Hedge does not admit improved regret under Bernstein conditions, which are standard low-noise conditions from statistical learning \citep{mammen1999margin,tsybakov2004aggregation,bartlett2006empirical}.
  By contrast, it was shown by \citet{koolen2016combining} that algorithms which satisfy some adaptive adversarial regret bound achieve improved regret under Bernstein conditions.
  Finally, we characterize the behavior of Decreasing Hedge in the stochastic regime, by showing that its eventual regret on \emph{any} stochastic instance is governed
  by the gap $\Delta$.

\paragraph{Related work.}

In the bandit setting, where the feedback only consists of the loss of the selected action, there has also been some interest in ``best-of-both-worlds'' algorithms that combine optimal $O (\sqrt{M T})$ worst-case regret in the adversarial regime with improved $O (M \log T)$ regret (up to logarithmic factors) in the stochastic case \citep{bubeck2012best2worlds,seldin2014practical,auer2016both}.
In particular, \citet{seldin2014practical,seldin2017improved} showed that by augmenting the standard EXP3 algorithm for the adversarial regime (an analogue of Hedge with $\Theta (1/\sqrt{t})$ learning rate)
with a special-purpose gap detection mechanism, one can achieve poly-logarithmic regret in the stochastic case.
This result is strengthened in some recent follow-up work \citep{zimmert2019optimal,zimmert2019beating},
that appeared since the completion of the first version of the present paper,
which obtains optimal regret in the stochastic and adversarial regimes through a variant of the Follow-The-Regularized-Leader (FTRL) algorithm with $\Theta (1/\sqrt{t})$ learning rate and a proper regularizer choice.
This result can be seen as an analogue in the bandit case of our upper bound for Decreasing Hedge.
Note that, in the bandit setting, the hardness of an instance is essentially characterized by the gap $\Delta$ \citep{bubeck2012regret}; in particular, the Bernstein condition, which depends on the correlations between the losses of the experts, cannot be exploited under bandit feedback, where one only observes one arm at each round.
Hence, it appears that the negative part of our results (on the limitations of Hedge) does not have an analogue in the bandit case.

A similar adaptivity result for FTRL with decreasing $\Theta (1/\sqrt{t})$ learning rate has been observed in a different context by \citet{huang2017curved}. Specifically, it is shown that, in the case of online linear optimization on a Euclidean ball, FTRL with squared norm regularizer and learning rate $\Theta (1/\sqrt{t})$ achieves $O (\log T)$ regret when the loss vectors are \iid
This result is an analogue of our upper bound for Hedge, since this algorithm corresponds to FTRL on the simplex with entropic regularizer \citep{cesabianchi2006plg,hazan2016online}.
On the other hand, the simplex lacks the curvature of the Euclidean ball, which is important to achieve small regret; here, the improved regret is ensured by a condition on the distribution, namely the existence of a gap $\Delta$.
Our lower bound for Hedge shows that this condition is necessary, thereby characterizing the long-term regret of FTRL on the simplex with entropic regularizer.
In the case of the Euclidean ball with squared norm regularizer, the norm of the expected loss vector appears to play a similar role, as shown by the upper bound from \citet{huang2017curved}.

\paragraph{Outline.}

We define the setting of prediction with expert advice and the Hedge algorithm in Section~\ref{sec:expert-problem-hedge}, and we recall herein its standard worst-case regret bound.
In Section~\ref{sec:hedge-easy}, we consider the behavior of the Hedge algorithm on easier instances, namely the stochastic setting with a gap $\Delta$ on the best expert. 
Under an i.i.d assumption on the sequence of losses, we provide in Theorem~\ref{thm:hedge-stochastic} an upper bound on the regret of order $(\log M) / \Delta$ for Decreasing Hedge. 
In Proposition~\ref{prop:lowerbound-gap}, we prove that the rate $(\log M) / \Delta$ cannot be improved in this setting.
In Theorem~\ref{thm:adv-gap} and  Corollary~\ref{cor:hedge-martingale}, we extend the regret guarantees to the adversarial with a gap setting, where a leading expert linearly outperforms the others.
These results stand for any Hedge algorithm which is worst-case optimal and with any learning rate which is larger than the one of Decreasing Hedge, namely $O(\sqrt{\log M/ t})$.
In Proposition~\ref{prop:lower-bound-hedge-cst}, we prove the sub-optimality of the fixed-horizon Hedge algorithm, and of another version of Hedge based on the so-called ``doubling trick''.
In Section~\ref{sec:advant-second-order}, we discuss the advantages of adaptive Hedge algorithms, and explain what the limitations of Decreasing Hedge are compared to such versions.
We include numerical illustrations of our theoretical findings in Section~\ref{sec:experiments}, conclude in Section~\ref{sec:conclusion} and provide the proofs in Section~\ref{sec:proofs}.

\section{The expert problem and the Hedge algorithm}
\label{sec:expert-problem-hedge}

In the Hedge setting, also called \emph{decision-theoretic online learning} \citep{freund1997boosting}, the learner and its adversary (the Environment) sequentially compete on the following game: at each round $t\geq 1$,
\begin{enumerate}
\item the Learner chooses a probability vector $\bm v_t = (v_{i,t})_{1\leq i \leq M}$ on the $M$ experts $1, \dots, M$;
\item the Environment picks a bounded loss vector $\bm \ell_t = (\ell_{i,t})_{1\leq i \leq M} 
\in [0, 1]^M$, where $\ell_{i, t}$ is the loss of expert $i$ at round $t$, while the Learner suffers loss $\wh \ell_t = \bm v_t^\top \bm \ell_t$.
\end{enumerate}
The goal of the Learner is to control its \emph{regret}
\begin{equation}
  \label{eq:regret}
  R_T = \sum_{t=1}^T \wh \ell_t - \min_{1\leq i \leq M} \sum_{t=1}^T \ell_{i,t}
\end{equation}
for every $T \geq 1$, irrespective of the sequence of loss vectors $\bm \ell_1, \bm \ell_2, \dots$ chosen by the Environment.
One of the most standard algorithms for this setting is the \emph{Hedge} algorithm.
The Hedge algorithm, also called the exponentially weighted averaged forecaster, uses the vector of probabilities $\bm v_t = (v_{i,t})_{1\leq i \leq M}$ given by
\begin{equation}
  \label{eq:hedge-algorithm}
  v_{i, t} = \frac{e^{-\eta_t L_{i,t-1}}}{\sum_{j=1}^M e^{-\eta_t L_{j,t-1}}}
\end{equation}
at each $t\geq 1$, where $L_{i,T} = \sum_{t=1}^T \ell_{i,t}$ denotes the cumulative loss of 
expert~$i$ for every $T\geq 1$.
Let us also denote $\wh L_T := \sum_{t=1}^T \wh \ell_t$ and $R_{i,T} = \wh L_T - L_{i,T}$ the regret with respect to expert $i$.
We consider in this paper the following variants of Hedge, where $c_0 > 0$ is a constant.

\medskip\noindent
\textbf{Decreasing Hedge}~\citep{auer2002adaptive}.
This is Hedge with the sequence of learning rates $\eta_t = c_0 \sqrt{\log(M) / t}$. 

\medskip\noindent
\textbf{Constant Hedge}~\citep{littlestone1994weighted}.
Given a finite time horizon $T \geq 1$, this is Hedge with constant learning rate $\eta_t = c_0 \sqrt{\log(M) / T}$.

\medskip\noindent
\textbf{Hedge with doubling trick}~\citep{cesabianchi1997doublingtrick, cesabianchi2006plg}.
This variant of Hedge uses a constant learning rate on geometrically increasing intervals, restarting the algorithm at the beginning of each interval. Namely, it uses
\begin{equation}
  \label{eq:hedge-doubling-trick}
  v_{i,t} = \frac{\exp ( - \eta_t \sum_{s = T_k}^{t-1} \ell_{i,s})}{\sum_{j=1}^M \exp ( - \eta_t \sum_{s = T_k}^{t-1} \ell_{j,s})},
\end{equation}
with $T_l = 2^l$ for $l \geq 0$, $k \in \N$ such that $T_k \leq t < T_{k+1}$ and
$\eta_t = c_0 \sqrt{\log (M) / T_k}$.

\medskip
Let us recall the following standard regret bound for the Hedge algorithm from \citet{chernov2010prediction}.
\begin{proposition}
  \label{prop:hedge-adversarial}
  Let $\eta_1, \eta_2, \dots$ be a decreasing sequence of learning rates.
  The Hedge algorithm \eqref{eq:hedge-algorithm} satisfies the following regret bound:
  \begin{equation}
    \label{eq:hedge-adversarial}
    R_T \leq \frac{1}{\eta_T} \log M + \frac{1}{8} \sum_{t=1}^T \eta_t
    \, .
  \end{equation}
  In particular, the choice $\eta_t = 2 \sqrt{{\log (M)}/{t}}$ yields a regret bound of $\sqrt{T \log M}$ for every $T \geq 1$.
\end{proposition}

Note that the regret bound stated in Equation~\eqref{eq:hedge-adversarial} holds for every sequence of losses $\bm \ell_1, \bm \ell_2, \dots$, which makes it valid under no assumption (aside from the boundedness of the losses).
The worst-case regret bound in $O(\sqrt{T \log M})$ is achieved by Decreasing Hedge, Hedge with doubling trick and Constant Hedge (whenever $T$ is known in advance).
The $O(\sqrt{T \log M})$ rate cannot be improved either by Hedge or any other algorithm: it is known to be the minimax optimal regret \citep{cesabianchi2006plg}.
Contrary to Constant Hedge, Decreasing Hedge is anytime, in the sense that it achieves the $O(\sqrt{T \log M})$ regret bound simultaneously for each $T \geq 1$.
We note that this worst-case regret analysis fails to exhibit any difference between these three algorithms.

In many cases, this $\sqrt{T}$ regret bound is pessimistic, and more ``aggressive'' strategies (such as the follow-the-leader algorithm, which plays at each round the uniform distribution on the experts with smallest loss, \citealp{cesabianchi2006plg}) may achieve constant regret in easier instances, even though they lack regret guarantees in the adversarial regime.
We show in Section~\ref{sec:hedge-easy} below that Decreasing Hedge is actually
better than both Constant Hedge and Hedge with doubling trick in some easier instance of the problem (including in the stochastic setting).
This entails that Decreasing Hedge is actually able to adapt, without any modification, to the easiness of the problem considered.

\section{Regret of Hedge variants on easy instances}
\label{sec:hedge-easy}

In this section, we depart from the worst-case regret analysis and study the regret of the considered variants of the Hedge algorithm on easier instances of the prediction with expert advice problem.

\subsection{Optimal regret for Decreasing Hedge in the stochastic regime}
\label{sub:optimal-stoch}

We examine the behavior of Decreasing Hedge in the stochastic regime, where the losses are the realization of some (unknown) stochastic process.
More precisely, we consider the standard \iid case, where the loss vectors $\bm \ell_1, \bm \ell_2, \dots$ are \iid (independence holds over rounds, but not necessarily across experts).
In this setting, the regret can be much smaller than the worst-case $\sqrt{T \log M}$ regret, since the best expert (with smallest expected loss) will dominate the rest after some time.
Following \citet{gaillard2014secondorder,luo2015adanormalhedge}, the easiness parameter we consider in this case, which governs the time needed for the best expert to have the smallest cumulative loss and hence the incurred regret, is the sub-optimality gap $\Delta = \min_{i \neq i^*} \E [\ell_{i,t} - \ell_{i^*,t} ]$, where $i^* = \argmin_{i} \E [ \ell_{i,t} ]$.

We show below  that, despite the fact that Decreasing Hedge is designed for the worst-case setting described in Section~\ref{sec:expert-problem-hedge}, it is able to {adapt} to the easier problem considered here, 
Indeed, Theorem~\ref{thm:hedge-stochastic} shows that Decreasing Hedge achieves a \emph{constant}, and in fact \emph{optimal} (by Proposition~\ref{prop:lowerbound-gap} below) regret bound in this setting, in spite of its ``conservative'' learning rate.

With the exception of the high-probability bound of Corollary~\ref{cor:hedge-martingale}, the upper and lower bounds in the stochastic case are stated for the \emph{pseudo-regret} $\pseudoregret_T = \E [R_{i^*,T}]$ (similar bounds hold for the the expected regret $\E [R_T]$, since $\pseudoregret_T \leq \E [R_T]$ and by Remark~\ref{rem:pseudoregret} in Section~\ref{sec:proof-theorem-1}).

\begin{theorem}
  \label{thm:hedge-stochastic}
  Let $M \geq 3$.
  Assume that the loss vectors $\bm \ell_1, \bm \ell_2, \dots$ are \iid random variables, where $\bm \ell_t = (\ell_{i,t})_{1\leq i \leq M}$.
  Also, assume that there exists $i^* \in \{1, \dots, M\}$ and $\Delta > 0$ such that
  \begin{equation}
    \label{eq:gap-condition}
    \E [ \ell_{i,t} - \ell_{i^*,t} ] \geq \Delta
  \end{equation}
  for every $i \neq i^*$.
  Then, the Decreasing Hedge algorithm with learning rate $\eta_t = 2 \sqrt{(\log M)/t}$
  achieves the following pseudo-regret bound\textup: for every $T \geq 1$\textup,
  \begin{equation}
    \label{eq:regret-stochastic-exp}
    \pseudoregret_T \leq \frac{4 \log M + 25}{\Delta}
    \, .
  \end{equation}
\end{theorem}

The proof of Theorem~\ref{thm:hedge-stochastic} is given in Section~\ref{sec:proof-theorem-1}.
Theorem~\ref{thm:hedge-stochastic} proves that, in the stochastic setting with a gap $\Delta$, the Decreasing Hedge algorithm achieves a regret $O(\log (M) / \Delta)$, without any prior knowledge of $\Delta$.
This matches the guarantees of adaptive Hedge algorithms which are explicitly designed to adapt to easier instances \citep{gaillard2014secondorder,luo2015adanormalhedge}.
This result may seem surprising at first: indeed, adaptive exponential weights algorithms
that combine optimal regret in the adversarial setting and constant regret in
easier scenarios, such as Hedge with a second-order tuning \citep{cesabianchi2007secondorder} or AdaHedge \citep{derooij2014followtheleader}, typically use a data-dependent learning rate $\eta_t$ that adapts to the properties of the losses.
While the learning rate $\eta_t$ chosen by these algorithms may be as low as the worst-case tuning $\eta_t \propto \sqrt{\log (M) / t}$, in the stochastic case those algorithms will use larger, lower-bounded learning rates to ensure constant regret.
As Theorem~\ref{thm:hedge-stochastic} above shows, it turns out that the data-independent, ``safe'' learning rates $\eta_t \propto \sqrt{\log (M) / t}$ used by ``vanilla'' Decreasing Hedge are still large enough to adapt to the stochastic case.

\paragraph{Idea of the proof.} 

The idea of the proof of Theorem~\ref{thm:hedge-stochastic} is to divide time in two phases: a short initial phase $\iint 1{t_1}$,
where $t_1 = O (\frac{\log M}{\Delta^2})$,
and a second phase $\iint{t_1}{T}$.
The initial phase is dominated by noise, and regret during this period is bounded
through the worst-case regret bound of Proposition~\ref{prop:hedge-adversarial}, which gives a regret of $O(\sqrt{t_1 \log M}) = O (\frac{\log M}{\Delta})$.
In the second phase, the best expert dominates the rest, and the weights concentrate on this best expert fast enough that the total regret incurred is small.
The control of the regret in the second phase relies on the critical fact that, if $\eta_t$ is at least as large as $\sqrt{(\log M)/t}$, then the following two things occur simultaneously at $t_1 \asymp \frac{\log M}{\Delta^2}$, namely at the beginning of the late phase:
\begin{enumerate}
\item with high probability, the best expert $i^*$ dominates all the others linearly: for every $i \neq i^*$ and $t \geq t_1$, $L_{i,t} - L_{i^*,t} \geq \frac{\Delta t}{2}$;
\item the total weight of all suboptimal experts is controlled: $\sum_{i \neq i^*} v_{i,t_1} \leq \frac{1}{2}$. If $\eta_t \geq \sqrt{(\log M)/t}$ and the first condition holds, this amounts to $M \exp (- \frac{\Delta}{2} \sqrt{t \log M}) \leq \frac{1}{2}$, namely $t_1 \gtrsim \frac{\log M}{\Delta^2}$.
\end{enumerate}
In other words, the learning rate $\eta_t \asymp \sqrt{(\log M)/t}$ ensures that the total weight of suboptimal experts starts vanishing at about the same time as when the best expert starts to dominate the others with a large probability (and remarkably, this property holds for every value of the sub-optimality gap $\Delta$).
Finally, the upper bound on the regret in the second phase rests on the two conditions above, together with the bound $\sum_{t \geq 1} e^{-c \sqrt{t}} = O (\frac{1}{c^2})$ for $c > 0$.

\begin{remark}
  The fact that $\sum_{t \geq 1} e^{-c \sqrt{t}} = O (1 / c^2)$ is also used in the analysis of the EXP3++ bandit algorithm~\citep[Lemma 10]{seldin2014practical}.
  In the expert setting considered here, summing the contribution of all experts (which suffices in the bandit setting to obtain the correct order of regret) would yield a significantly suboptimal $O(M / \Delta)$ regret bound, with a linear dependence on the number of experts $M$.
  In our case, the decomposition of the regret in two phases, which is explained above, removes the linear dependence on $M$ and allows to obtain the optimal rate $(\log M) / \Delta$.
\end{remark}

We complement Theorem~\ref{thm:hedge-stochastic} by showing that the $O((\log M) / \Delta)$ regret under the gap condition cannot be improved, in the sense that its dependence on both $M$ and $\Delta$ is optimal.

\begin{proposition}
  \label{prop:lowerbound-gap}
  Let $\Delta \in (0, \frac{1}{4})$, $M \geq 4$ and $T \geq (\log M) / (16 \Delta^2)$.
  Then, for any algorithm for the Hedge setting, there exists an \iid distribution over the sequence of losses $(\bm \ell_{t})_{t \geq 1}$ such that\textup:
  \begin{itemize}
  \item there exists $i^* \in \{ 1, \dots, M \}$ such that, for any $i \neq i^*$, $\E [\ell_{i,t} - \ell_{i^*, t}] \geq \Delta$\textup;
  \item the pseudo-regret of the algorithm satisfies\textup:
    \begin{equation}
      \label{eq:lowerbound-gap}
      \pseudoregret_T
      \geq \frac{\log M}{256 \Delta}
      \, .
    \end{equation}
  \end{itemize}
\end{proposition}

The proof of Proposition~\ref{prop:lowerbound-gap} is given in Section~\ref{sec:proof-lowerbound-gap}.
Proposition~\ref{prop:lowerbound-gap} generalizes the well-known minimax lower bound of $\Theta (\sqrt{T \log M})$, which is recovered by taking $\Delta \asymp \sqrt{(\log M)/T}$.

\subsection{Small regret for Decreasing Hedge in the adversarial with a gap problem}
\label{sub:adv-gap}

In this section, we extend the regret guarantee of Decreasing Hedge in the stochastic setting (Theorem~\ref{thm:hedge-stochastic}), by showing that it holds for more general algorithms and under more general assumptions.
Specifically, we consider an ``adversarial with a gap'' regime, similar to the one introduced by \citet{seldin2014practical} in the bandit case, where the leading expert linearly outperforms the others after some time.
As Theorem~\ref{thm:adv-gap} shows, essentially the same regret guarantee can be obtained in this case, up to an additional $\log (\Delta^{-1}) / \Delta$ term.
Theorem~\ref{thm:adv-gap} also applies to any Hedge algorithm whose (possibly data-dependent) learning rate $\eta_t$ is at least as large as that of Decreasing Hedge, and which satisfies a $O(\sqrt{T \log M})$ worst-case regret bound;
this includes algorithms with \emph{anytime} first and second-order tuning of the learning rate \citep{auer2002adaptive,cesabianchi2007secondorder,derooij2014followtheleader}.
In what follows, we will assume $M \geq 3$ for convenience; similar results holds for $M=2$.
\begin{theorem}
  \label{thm:adv-gap}
  Let $M \geq 3$.
  Assume that there exists $\tau_0 \geq 1$\textup, $\Delta \in (0, 1)$ and $i^* \in \{1, \dots, M \}$ 
  such that\textup, for every $t \geq \tau_0$ and $i \neq i^*$\textup, one has
  \begin{equation}
    \label{eq:condition-gap-adv}
    L_{i, t} - L_{i^*, t} \geq \Delta t.
  \end{equation}
  Consider any Hedge algorithm with \textup(possibly data-dependent\textup) learning rate $\eta_t$ such that  
  \begin{itemize}
  \item $\eta_t \geq c_0 \sqrt{(\log M) / t}$ for some constant $c_0 > 0$\textup;
  \item it admits the following worst-case regret bound: $R_T \leq c_1 \sqrt{T \log M}$ for every $T \geq 1$\textup,
    for some $c_1 > 0$.
  \end{itemize}
  Then, for every $T\geq 1$, the regret of this algorithm is upper bounded as
  \begin{equation}
    \label{eq:regret-gap-adv}
    R_T 
    \leq c_1 \sqrt{\tau_0 \log M} + \frac{c_2 \log M + c_3 \log {\Delta}^{-1} + c_4}{\Delta}    
  \end{equation}
  where $c_2 = c_1 + \frac{\sqrt{8}}{c_0}$, $c_3 = \frac{\sqrt{8}}{c_0}$ and $c_4 = \frac{16}{c_0^2}$.
\end{theorem}

The idea of the proof of Theorem~\ref{thm:adv-gap} is the same as that of Theorem~\ref{thm:hedge-stochastic}, the only difference being the slightly longer initial phase to account for the adversarial nature of the losses.
As a consequence of the general bound of Theorem~\ref{thm:adv-gap}, we can recover the guarantee of Theorem~\ref{thm:hedge-stochastic} (up to an additional $\log (\Delta^{-1}) /\Delta$ term), both in expectation and with high probability, under more general stochastic assumptions than \iid over time.
The proofs of Theorem~\ref{thm:adv-gap} and Corollary~\ref{cor:hedge-martingale} are provided in Section~\ref{sec:proof-thm-adv-gap}.

\begin{corollary}
  \label{cor:hedge-martingale}
  Assume that the losses $(\ell_{i,t})_{1\leq i \leq M, t\geq 1}$ are random variables.  
  Also, denoting $\F_t = \sigma \big( (\ell_{i,s})_{1\leq i \leq M, 1\leq s \leq t} \big)$, assume that there exists $i^*$ and $\Delta >0 $ such that
  \begin{equation}
    \label{eq:gap-condition-martingale}
    \E \left[ \ell_{i,t} - \ell_{i^*,t} \cond \F_{t-1} \right] \geq \Delta
  \end{equation}
  for every $i\neq i^*$ and every $t\geq 1$.
  Then, for any Hedge algorithm satisfying the conditions of Theorem~\ref{thm:adv-gap}, and every $T \geq 1$\textup:
  \begin{equation}
    \label{eq:regret-martingale-exp}
    \pseudoregret_T
    \leq (5 c_1 + 2 c_2) \frac{\log M}{\Delta} + 2 c_3 \frac{\log \Delta^{-1}}{\Delta} 
    + \frac{2 c_4}{\Delta},
  \end{equation}
  with $c_1, c_2, c_3, c_4$ as in Theorem~\ref{thm:adv-gap}.
  In addition, for every $\eps \in (0, 1)$, we have
  \begin{equation}
    \label{eq:regret-martingale-prob}
    R_T
    \leq \left( c_1 \sqrt{8} + 2 c_2 \right) \frac{\log M}{\Delta} + c_1 \frac{\sqrt{8\log M \log \eps^{-1}}}{\Delta} + 2 c_3 \frac{\log \Delta^{-1}}{\Delta} + \frac{2 c_4}{\Delta}
  \end{equation}
  with probability at least $1 - \eps$.  
\end{corollary}

\subsection{Constant Hedge and Hedge with the doubling trick do not adapt to the stochastic case}
\label{sec:negative-results}

Now, we show that the adaptivity of Decreasing Hedge to gaps in the losses, established in Sections~\ref{sub:optimal-stoch} and~\ref{sub:adv-gap}, is not shared by the two closely related Constant Hedge and Hedge with the doubling trick, despite the fact that they both achieve the minimax optimal worst-case $O (\sqrt{T \log M})$ regret.
Proposition~\ref{prop:lower-bound-hedge-cst} below shows that both algorithms fail to achieve a constant regret, and in fact to improve over their worst-case $\Theta (\sqrt{T \log M})$ regret guarantee, even in the extreme case of experts with constant losses $0$ (for the leader), and $1$ for the rest (\ie, $\Delta = 1$).

\begin{proposition}
  \label{prop:lower-bound-hedge-cst}
  Let $T \geq 1$\textup, $M \geq 2$\textup, and consider the experts $i=1, \dots, M$ with losses $\ell_{1, t} = 0$\textup, 
  $\ell_{i, t} = 1$ $(1 \leq t \leq T, 2 \leq i \leq M)$.
  Then\textup, the pseudo-regret of Constant Hedge with learning rate $\eta_t = c_0 \sqrt{\log (M)/T}$ 
  \textup(where $c_0 > 0$ is a numerical constant\textup) is lower bounded as follows\textup:
  \begin{equation}
    \label{eq:lower-bound-hedge-cst}
    \pseudoregret_T
    \geq \min \Big( \frac{\sqrt{T \log M}}{3 c_0}, \frac{T}{3} \Big)
    \, .
  \end{equation}
  In addition, Hedge with doubling trick~\eqref{eq:hedge-doubling-trick} also suffers a pseudo-regret satisfying
  \begin{equation}
    \label{eq:lower-bound-hedge-doubling}
    \pseudoregret_T \geq
    \min \Big ( \frac{\sqrt{T \log M}}{6 c_0} , \frac{T}{12} \Big)
    \, .
  \end{equation}
\end{proposition}

The proof of Proposition~\ref{prop:lower-bound-hedge-cst} is given in Section~\ref{sec:proof-lower-bound-hedge-cst}.
Although Hedge with a doubling trick is typically considered as overly conservative and only suitable for worst-case scenarios \citealp{cesabianchi2006plg} (especially due to its periodic restarts, after which it discards past observations), to the best of our knowledge Proposition~\ref{prop:lower-bound-hedge-cst} (together with Theorem~\ref{thm:hedge-stochastic}) is the first to formally demonstrate the advantage of Decreasing Hedge over the doubling trick version.
This implies that Decreasing Hedge should not be seen as merely a substitute for Constant Hedge to achieve anytime regret bounds.
Indeed, even when the horizon $T$ is fixed, Decreasing Hedge outperforms Constant Hedge in the stochastic setting.

\section{Limitations of Decreasing Hedge in the stochastic case}
\label{sec:advant-second-order}

In this section, we explore the limitations of the simple Decreasing Hedge algorithm in the stochastic regime, and exhibit situations where it performs worse than more sophisticated algorithms.
The starting observation is that the sub-optimality gap $\Delta$ is a rather brittle measure of ``hardness'' of a stochastic instance, which does not fully reflect the achievable rates.
We therefore consider the following fast-rate condition from statistical learning, which refines the sub-optimality gap as a measure of complexity of a stochastic instance.

\begin{definition}[Bernstein condition]
  \label{def:bernstein-condition}
  Assume that the losses $\bm \ell_1, \bm \ell_2, \dots$ are the realization of a stochastic process.
  Denote $\F_{t} = \sigma (\bm \ell_1, \dots, \bm \ell_t)$ the $\sigma$-algebra generated by $\bm \ell_1, \dots, \bm \ell_t$.
  For $\beta \in [0,1]$ and $B >0$, the losses are said to satisfy the \emph{$(\beta,B)$-Bernstein condition} if there exists $i^*$ such that, for every $t \geq 1$ and $i \neq i^*$,
  \begin{equation}
    \label{eq:bernstein-condition}
    \E [(\ell_{i,t} - \ell_{i^*,t})^2 \cond \F_{t-1}]
    \leq B \E [ \ell_{i,t} - \ell_{i^*,t} \cond \F_{t-1} ]^\beta
    \, .        
  \end{equation}
\end{definition}

The Bernstein condition \citep{bartlett2006empirical}, a generalization of the Tsybakov margin condition \citep{tsybakov2004aggregation,mammen1999margin}, is a geometric property on the losses which enables to obtain fast rates (e.g., faster than $O(1/\sqrt{n})$ for parametric classes) in statistical learning; we refer to \citet{vanerven2015fastrates} for a discussion of fast rates conditions.
  The Bernstein condition~\eqref{eq:bernstein-condition}
  quantifies the 
  ``easiness'' of a stochastic instance, and generalizes the gap condition considered in the previous section (see Example~\ref{ex:bernstein-gap} below).
Roughly speaking, it states that good experts (with near-optimal expected loss) are highly correlated with the best expert.
In the examples below, we assume that the loss vectors $\bm \ell_1, \bm \ell_2, \dots$ are \iid

\begin{example}[Gap implies Bernstein]
  \label{ex:bernstein-gap}
  If $\Delta_i = \E [\ell_{i,t} - \ell_{i^*,t}] \geq \Delta$ for $i \neq i^*$, then the $(1, \frac{1}{\Delta})$-Bernstein condition holds \citep[Lemma~4]{koolen2016combining}.
  Furthermore, letting $\alpha = \E [\ell_{i^*,t}]$ denote the expected loss of the best expert, the $(1, 1 + \frac{2 \alpha}{\Delta})$-Bernstein condition holds.
  Indeed, for any $i \neq i^*$,
  denoting $\mu_i := \E [\ell_{i,t}] = \alpha + \Delta_i$, we have (since $(u-v)^2 \leq \max(u^2,v^2) \leq u^2+v^2 \leq u+v$ for $u,v \in [0, 1]$):
  \begin{align*}
    \E \big[ (\ell_{i,t} - \ell_{i^*,t})^2 \big]
      &\leq \E \left[ \ell_{i,t} + \ell_{i^*,t} \right] 
      = \frac{\mu_i + \alpha}{\mu_i - \alpha} \E \left[ \ell_{i,t} - \ell_{i^*,t} \right] 
      = \Big( 1 + \frac{2 \alpha}{\Delta_i} \Big) \, \E \left[ \ell_{i,t} - \ell_{i^*,t} \right]
        \, ,
  \end{align*}
  which establishes the claim
  since $\Delta_i \geq \Delta$.
  This provides an improvement when $\alpha$ is small.
\end{example}

\begin{example}[Bernstein without a gap]
  \label{ex:bernstein-without-gap}
  Let $P$ be a distribution on $\X \times \{ 0, 1\}$, where $\X$ is some measurable space. Assume that $(X_1, Y_1), (X_2, Y_2) \dots$ are \iid samples from $P$, and that the experts $i \in \{ 1, \dots, M\}$ correspond to classifiers $f_i : \X \to \{ 0, 1 \}$: $\ell_{i, t} = \bm 1 ( f_i (X_t) \neq Y_t ) $, and that expert $i^*$ is the Bayes classifier: $f_{i^*} (X) = \bm 1 ( \eta (X) \geq 1/2 ) $, where $\eta(X) = \P (Y=1 \cond X)$.
    Tsybakov's low noise condition \citep{tsybakov2004aggregation}, namely $\P ( | 2 \eta (X) - 1 | \leq t ) \leq C t^{\kappa}$ for some $C > 0$, $\kappa \geq 0$ and every $t > 0$, implies the $(\frac{\kappa}{\kappa + 1}, B)$-Bernstein condition for some $B$ (see, e.g., \citealp{boucheron2005survey}).
    In addition, under the Massart condition \citep{massart2006risk} that
    $| \eta (X) - 1/2 | \geq c > 0$, the $(1, 1/(2c))$-Bernstein condition holds.
  Note that these conditions may hold even with an arbitrarily small sub-optimality gap $\Delta$, since the $f_i$, $i \neq i^*$, may be arbitrary.
\end{example}

Theorem~\ref{thm:hedge-no-bernstein} below shows that Decreasing Hedge fails to achieve improved rates under Bernstein conditions.

\begin{theorem}
  \label{thm:hedge-no-bernstein}
  For every $T \geq 1$, there exists a $(1, 1)$-Bernstein stochastic instance on which the pseudo-regret of the Decreasing Hedge algorithm with $\eta_t = c_0 \sqrt{(\log M) / t}$ satisfies
  $\pseudoregret_{T} \geq \frac{1}{3} \min( \frac{1}{c_0} \sqrt{T \log M}, {T})$.
\end{theorem}

The proof of Theorem~\ref{thm:hedge-no-bernstein} is given in Section~\ref{sec:proof-lowerbound-no-bernstein}.
By contrast, it was shown by \citet{koolen2016combining}
(and implicitly used by \citealp{gaillard2014secondorder}) that
some adaptive algorithms with data-dependent regret bounds enjoy improved regret under the Bernstein condition.
For the sake of completeness, we state this fact in Proposition~\ref{prop:second-order-bernstein} below, which corresponds to~\citet[Theorem~2]{koolen2016combining}, but where the dependence on $B$ is made explicit. We also only provide a bound in expectation, which considerably simplifies the proof.
The proof of Proposition~\ref{prop:second-order-bernstein}, which uses the same ideas as \citet[Theorem~11]{gaillard2014secondorder}, is provided in Section~\ref{sec:proof-second-order-bernstein}.

\begin{proposition}
  \label{prop:second-order-bernstein}
  Consider an algorithm for the Hedge problem which satisfies the following regret bound: for 
  every $i\in \{ 1, \dots, M\}$\textup,
  \begin{equation}
    \label{eq:second-order-regret}
    R_{i,T} \leq C_1 \sqrt{(\log M) \sum_{t=1}^T (\wh \ell_t - \ell_{i,t})^2} + C_2 \log M
  \end{equation}
  where
  $C_1, C_2 >0$ are constants.
  Assume that the losses satisfy the $(\beta, B)$-Bernstein condition.
  Then, the pseudo-regret of the algorithm satisfies\textup:
  \begin{equation}
    \label{eq:regret-bernstein}
    \pseudoregret_T
    \leq C_3 (B \log M)^{\frac{1}{2-\beta}} T^{\frac{1-\beta}{2-\beta}} + C_4 \log M 
  \end{equation}
  where $C_3 = \max (1, 4C_1^2)$ and $C_4 = 2 C_2$.
\end{proposition}

The data-dependent regret bound~\eqref{eq:second-order-regret}, a ``second-order'' bound, is satisfied by adaptive algorithms such as Adapt-ML-Prod \citep{gaillard2014secondorder} and Squint \citep{koolen2015squint}.
A slightly different variant of second-order regret bounds, which depends on some cumulative variance of the losses across experts, has been considered by \citet{cesabianchi2007secondorder,derooij2014followtheleader}, and is achieved by Hedge algorithms with a data-dependent tuning of the learning rate. Second-order bounds refine so-called \emph{first-order} bounds \citep{cesabianchi1997doublingtrick,auer2002adaptive,cesabianchi2006plg}, which are adversarial regret bounds that scale as $O(\sqrt{L_T^* \log M} + \log M)$, where $L_T^*$ denotes the cumulative loss of the best expert.
  While first-order bounds may still scale as the worst-case $O(\sqrt{T \log M})$ rate in a typical stochastic instance (where the best expert has a positive expected loss), second-order algorithms are known to achieve constant $O((\log M) / \Delta)$ regret in the stochastic case with gap $\Delta$ \citep{gaillard2014secondorder,koolen2015squint}.

Theorem~\ref{thm:hedge-no-bernstein}, in light of Proposition~\ref{prop:second-order-bernstein}, clarifies where the advantage of second-order algorithms compared to Decreasing Hedge lies: unlike the latter, they can exploit Bernstein conditions on the losses.
The contrast is most apparent for Bernstein instances with $\beta = 1$.  
By Example~\ref{ex:bernstein-gap}, the existence of a gap $\Delta$ implies that the $(1,B)$-Bernstein condition holds with $B \leq \frac{1}{\Delta}$.
However, as shown by Example~\ref{ex:bernstein-without-gap}, $B$ can in fact be much smaller than $\Delta$, in which case the regret bound~\eqref{eq:regret-bernstein} satisfied by second-order algorithms, namely $O (B \log M)$, significantly improves over the upper bound of $O( (\log M)/\Delta)$ of Decreasing Hedge from Theorem~\ref{thm:hedge-stochastic}.
Theorem~\ref{thm:hedge-no-bernstein} provides an instance where the difference does occur, in the most pronounced case where $B =1$, so that second-order algorithms enjoy small $O (\log M)$ regret, while Decreasing Hedge suffers $\Theta (\sqrt{T \log M})$ regret.

\begin{remark}
  The advantage of larger learning rates on some stochastic instances may be understood intuitively as follows.
  Consider an instance with $B$ small but small gap $\Delta$.
  The learning rate of Decreasing Hedge is large enough that it can rule out bad experts (with large enough gap $\Delta_i$) at the optimal rate (\ie, at time $(\log M)/\Delta_i^2$).
  However, once these bad experts are ruled out, near-optimal experts (with small gap $\Delta_i$) are ruled out late (after $(\log M)/\Delta_i^2$ rounds).
  On the other hand, the Bernstein assumption entails that those experts are highly correlated with the best expert, the amount of noise on the relative losses of these near-optimal experts is small, so that a larger learning rate could be safely used and would enable to dismiss near-optimal experts sooner.
\end{remark}

Setting the Bernstein condition aside, we conclude by investigating the intrinsic limitations of Decreasing Hedge in the stochastic setting.
Indeed, it is natural to ask whether Decreasing Hedge can exploit some other regularity of a stochastic instance, apart from the gap $\Delta$.
Theorem~\ref{thm:hedge-characterize-gap} shows that this is in fact not the case.

\begin{theorem}
  \label{thm:hedge-characterize-gap}  
  For every \iid \textup(over time\textup) stochastic instance with a unique best expert 
  \begin{equation*}
     i^* = \argmin_{1 \leq i \leq M} \E [\ell_{i,t}],
  \end{equation*}
  the pseudo-regret of Decreasing Hedge \textup(with $c_0 \geq 1$\textup) satisfies
    \begin{equation*}
      \pseudoregret_T \geq \frac{1}{450 c_0^4 (\log M)^2 \Delta}
    \end{equation*}
  for $T \geq \frac{1}{4 \Delta^2}$\textup, where $\Delta := \inf_{i \neq i^*} \E [\ell_{i,t} - \ell_{i^*,t}]$.
\end{theorem}

Theorem~\ref{thm:hedge-characterize-gap} shows (together with the upper bound of Theorem~\ref{thm:hedge-stochastic}) that the eventual regret of Decreasing Hedge on \emph{any} stochastic instance is determined by the sub-optimality gap $\Delta$, and scales (up to a $\log^3 M$ factor, depending on the number of near-optimal experts) as $\Theta (\frac{1}{\Delta})$.
This characterizes the behavior of Decreasing Hedge on any stochastic instance.

\section{Experiments}
\label{sec:experiments}

In this section, we illustrate our theoretical results by numerical experiments that compare the behavior of various Hedge algorithms in the stochastic regime.

\paragraph{Algorithms.} We consider the following algorithms: \texttt{hedge} is Decreasing Hedge with the default learning rates $\eta_t = 2\sqrt{\log (M) / t}$, \texttt{hedge\_constant} is Constant Hedge with constant learning rate  $\eta_t = \sqrt{8 \log (M) / T}$, \texttt{hedge\_doubling} is Hedge with doubling trick with $c_0 = \sqrt{8}$, \texttt{adahedge} is the AdaHedge algorithm from \citet{derooij2014followtheleader}, which is a variant of the Hedge algorithm with a data-dependent tuning of the learning rate $\eta_t$ (based on $\bm \ell_1, \dots, \bm \ell_{t-1}$).
As shown in the note \citet{blog}, AdaHedge also benefits from Bernstein conditions.
A related algorithm, namely Hedge with second-order tuning of the learning rate \citep{cesabianchi2007secondorder}, performed similarly to AdaHedge on the examples considered below, and was therefore not included. \texttt{FTL} is Follow-the-Leader \citep{cesabianchi2006plg} which puts all mass on the expert with the smallest loss (breaking ties randomly).
While FTL serves as a benchmark in the stochastic setting, unlike the other algorithms it lacks any guarantee in the adversarial regime, where its worst-case regret is \emph{linear} in $T$.

\paragraph{Results.}

We report in Figure~\ref{fig:experiments} the cumulative regrets of the considered algorithms in four examples.
The results for the stochastic instances (a), (b) and (c) described below are averaged over $50$ trials.

\medskip
\noindent
\emph{\textup(a\textup) Stochastic instance with a gap.}
This is the standard instance considered in this paper.
The losses are drawn independently from Bernoulli distributions (one of parameter $0.3$, $2$ of parameter $0.4$ and $7$ of parameter $0.5$, so that $M=10$ and $\Delta = 0.1$).
The results of Figure~\ref{fig:1a} confirm our theoretical results: Decreasing Hedge achieves a small, constant regret which is close to that of AdaHedge and FTL, while Constant Hedge and Hedge with doubling trick suffer a larger regret of order $\sqrt{T}$ (note that, although the expected regret of Constant Hedge converges in this case, the value of this limit depends on its learning rate and hence on $T$).

\medskip
\noindent
\emph{\textup(b\textup) ``Hard'' stochastic instance.}
This example has a zero gap $\Delta = 0$ between the two leading experts and $M=10$, which makes it ``hard'' from the standpoint of Theorem~\ref{thm:hedge-stochastic} (which no longer applies in this limit case).
The losses are drawn from independent Bernoulli distributions, of parameters $0.5$ for the $2$ leading experts, and $0.7$ for the $8$ remaining ones.
Although all algorithms suffer an unavoidable $\Theta (\sqrt{T})$ regret due to pure noise, Decreasing Hedge, AdaHedge and FTL achieve better regret than the two conservative Hedge variants (Figure~\ref{fig:1b}).
This is due to the fact that for the former algorithms, the weights of suboptimal experts decrease quickly and only induce a constant regret.

\medskip
\noindent
\emph{\textup(c\textup) Small loss for the best expert.}
In this experiment, we illustrate one advantage of adaptive Hedge algorithms such as AdaHedge over Decreasing Hedge, namely the fact that they admit improved regret bounds when the leading expert has small loss. We considered in this experiment $M = 10$, $\Delta = 0.04$ and the leading expert is $\betadist(0.04,0.96)$, then $4$ $\betadist(0.08,0.92)$, then $5$ $\betadist(0.5, 0.5)$.

\medskip
\noindent
\emph{\textup(d\textup) Adversarial with a gap instance.}
This simple instance is not random, and satisfies the assumptions of Theorem~\ref{thm:adv-gap}.
It is defined by $M=3$, $\Delta = 0.04$, $\ell_{3, t} = \frac{3}{4}$ for $t \geq 1$, $(\ell_{1, t}, \ell_{2,t}) = (\frac{1}{2}, 0)$ if $t=1$, $(0, 1)$ if $t \geq 80$ or if $t$ is even, and $(1, 0)$ otherwise.
FTL suffers linear regret in the first phase, while Constant Hedge and Hedge with doubling trick suffer $\Theta (\sqrt{T})$ during the second phase.

\begin{figure}
\centering
\begin{subfigure}[b]{.48\linewidth}
  \centering
  \includegraphics[width=\linewidth]{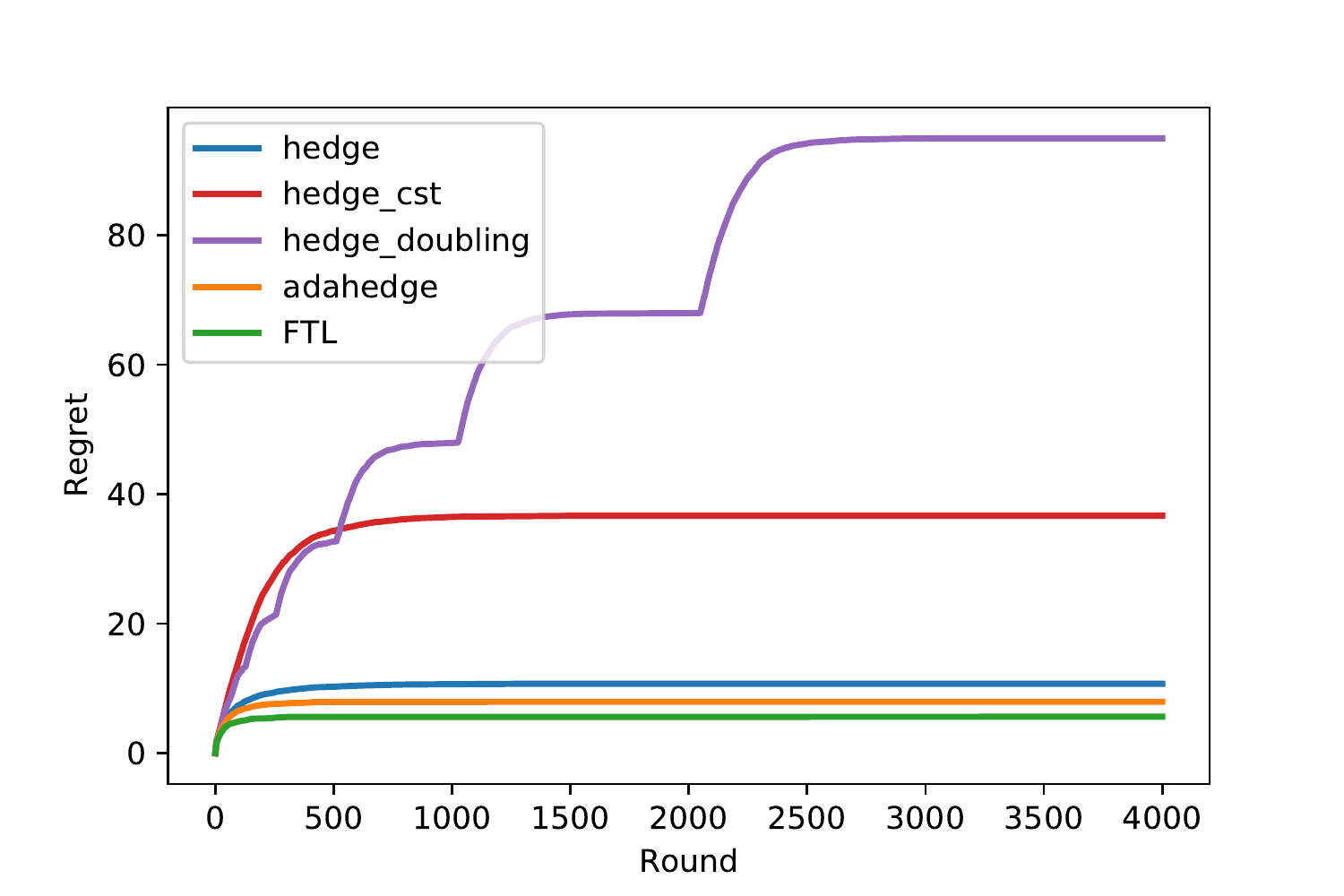}
  \caption{
    }
  \label{fig:1a}
\end{subfigure}%
\begin{subfigure}[b]{.48\linewidth}
  \centering
  \includegraphics[width=\linewidth] {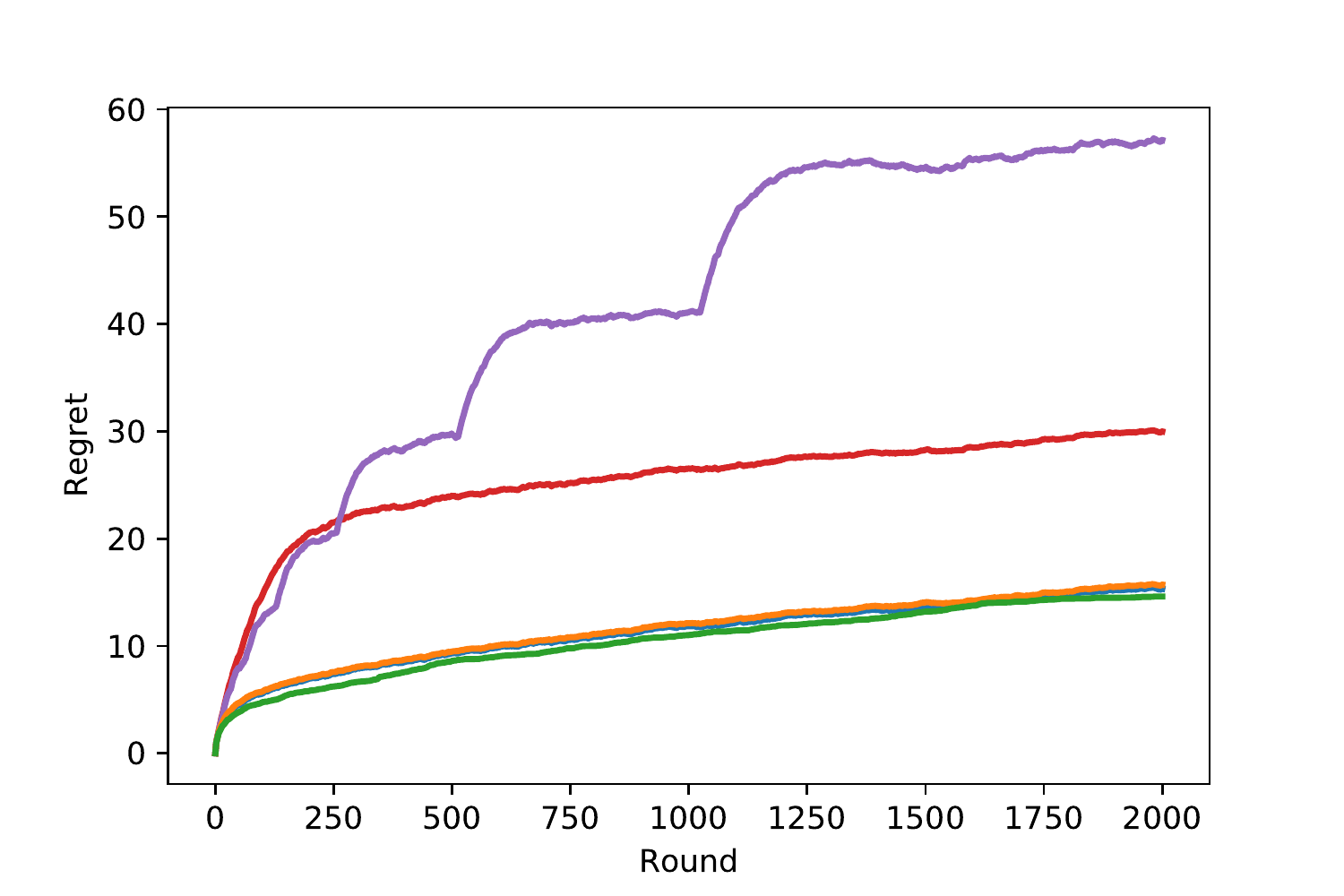}
  \caption{
  }\label{fig:1b}
\end{subfigure} \\ %
\begin{subfigure}[b]{.48\linewidth}
  \centering
  \includegraphics[width=\linewidth] {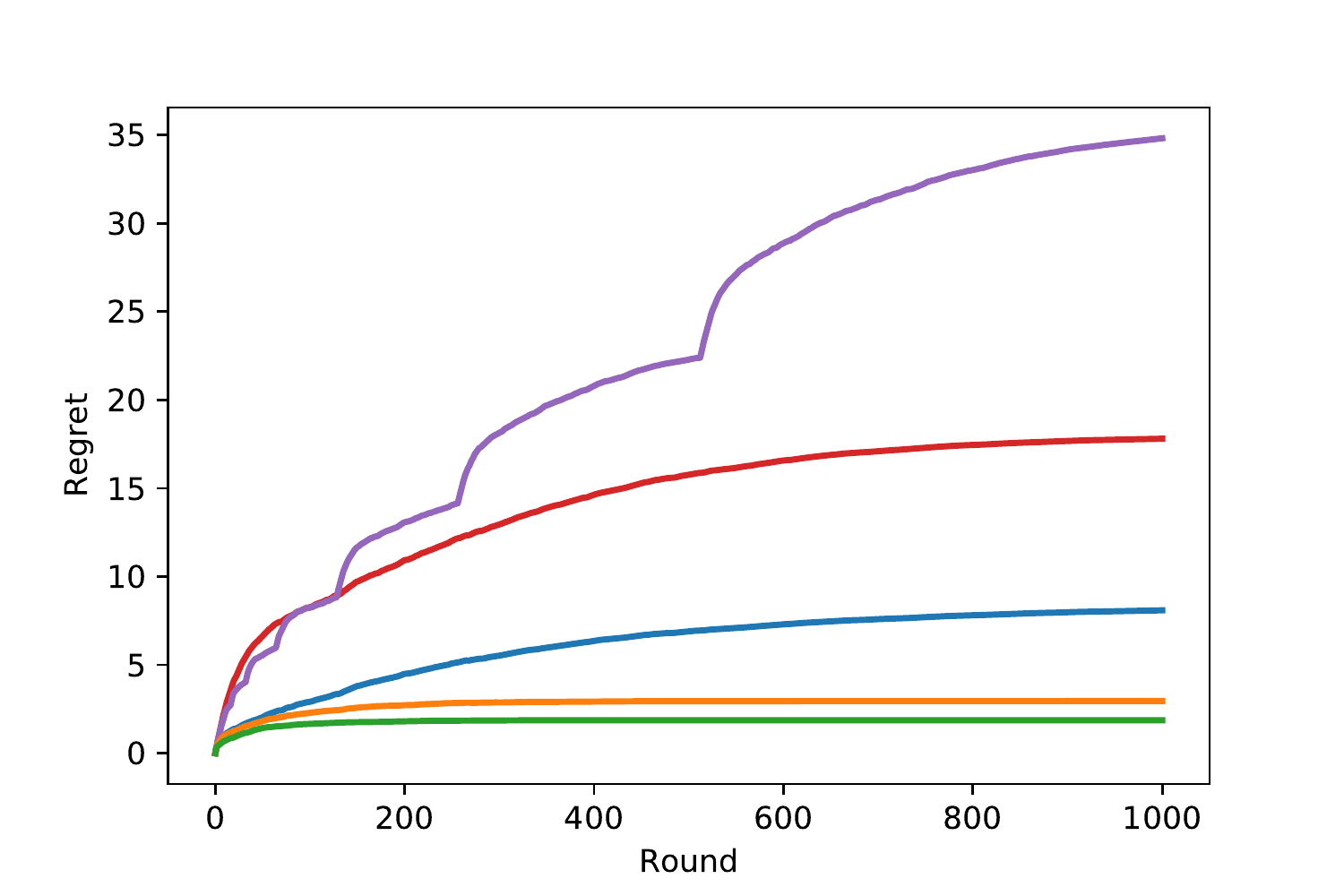}
  \caption{
  }\label{fig:1c}
\end{subfigure}%
\begin{subfigure}[b]{.48\linewidth}
  \includegraphics[width=\linewidth] {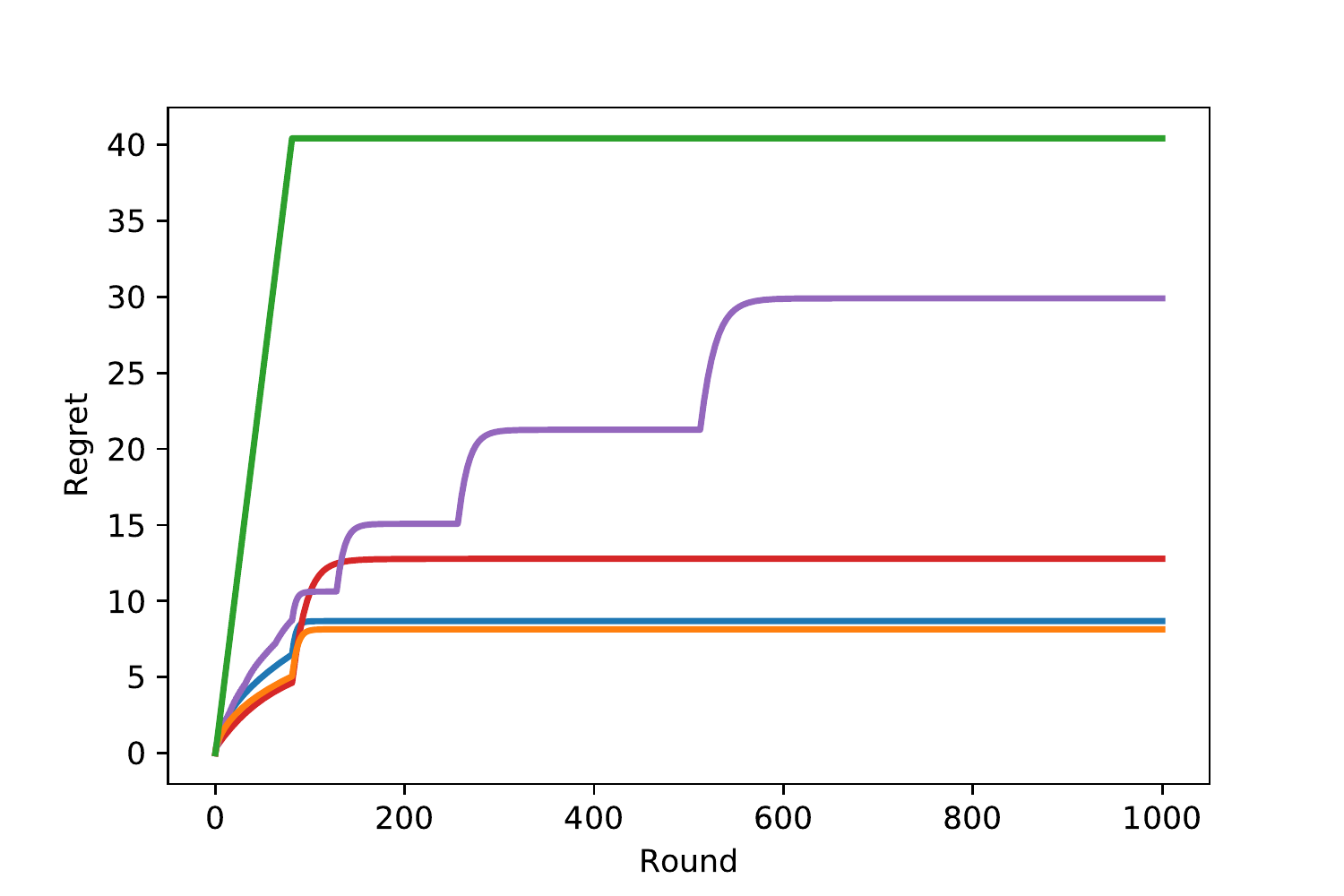}  
  \caption{
  }\label{fig:1d}
\end{subfigure}
\caption{Cumulative regret of Hedge algorithms on four examples, see text for a precise description and discussion about the results. (a) Stochastic instance with a gap; (b) ``Hard'' stochastic instance; (c) Small loss for the best expert; (d) Adversarial with a gap instance.}
\label{fig:experiments}
\end{figure}

\section{Conclusion} 
\label{sec:conclusion}

In this article, we carried the regret analysis of the standard exponential weights (Hedge) algorithm in the stochastic expert setting, closing a gap in the existing literature.
Our analysis reveals that, despite being tuned for the worst-case adversarial setting and lacking any adaptive tuning of the learning rate, Decreasing Hedge achieves optimal regret in the stochastic setting.
This property also enables one to distinguish it qualitatively from other variants including the one with fixed (horizon-dependent) learning rate or the one with doubling trick, which both fail to adapt to gaps in the losses.
To the best of our knowledge, this is the first result that shows the superiority of the decreasing learning rate over the doubling trick.
In addition, it suggests that, even for a fixed time horizon $T$, the decreasing learning rate tuning should be favored over the constant one.

Finally, we showed that the regret of Decreasing Hedge on any stochastic instance is essentially characterized by the sub-optimality gap $\Delta$.
This shows that adaptive algorithms, including algorithms achieving second-order regret bounds, can actually outperform Decreasing Hedge on some stochastic instances that exhibit a more refined form of ``easiness''.

\paragraph{A link with stochastic optimization.} 
\label{par:connection_with_stochastic_optimization}

Our results have a similar flavor to a well-known result \citep{moulines2011nonasymptotic} about stochastic optimization: stochastic gradient descent (SGD) with learning rate $\eta_t \propto 1 / \sqrt{t}$ (which is tuned for the convex case but not for the non-strongly convex case) and Polyak-Ruppert averaging achieves a fast $O(1 / (\mu t))$ excess risk rate for $\mu$-strongly convex problems, without the knowledge of $\mu$.
However, this link stops here since the two results are of a significantly different nature: the $O(1 / (\mu t))$ rate is satisfied only by SGD with Polyak-Ruppert averaging, and it does not come from a regret bound;
even in the $\mu$-strongly convex case, it can be seen that SGD with step-size $\eta_t \propto 1 / \sqrt{t}$ suffers a $\Theta (\sqrt{t})$ regret.
In fact, the opposite phenomenon occurs: in stochastic optimization, SGD uses a \emph{larger} $\Theta(1/\sqrt{t})$ step-size than the $\Theta(1/(\mu t))$ step size which exploits the knowledge of strong convexity, but the effect of this larger step-size is balanced by the averaging.
By contrast, in the expert setting, Hedge uses a \emph{smaller} $\Theta(\sqrt{(\log M)/t})$ learning rate than the constant, large enough learning rate which exploits the knowledge of the stochastic nature of the problem.

\paragraph{Acknowledgments.} 

The authors wish to warmly thank all four Anonymous Reviewers for their helpful feedback and insights on this work.
In particular, the proof of Proposition~\ref{prop:lowerbound-gap} was proposed by an Anonymous Referee, which allowed to shorten our initial proof.

\section{Proofs}
\label{sec:proofs}

We now provide the proofs of the results from the previous Sections, by order of appearance in the text.

\subsection{Proof of Theorem~\ref{thm:hedge-stochastic}}
\label{sec:proof-theorem-1}

Let $t_0 = \left\lceil \frac{8\log M}{\Delta^2} \right\rceil$, so that $\sqrt{t_0} \leq \sqrt{1 + \frac{8 \log M}{\Delta^2}} \leq 1 + \frac{\sqrt{8 \log M}}{\Delta}$ (since $\sqrt{a+b} \leq \sqrt{a} + \sqrt{b}$ for $a,b \geq 0$).
The worst-case regret bound of Hedge (Proposition~\ref{prop:hedge-adversarial}) shows that for $1 \leq T \leq t_0$:
\begin{equation}
  R_{i^*,T}
  \leq \sqrt{T \log M} 
  \leq \sqrt{t_0 \log M} 
  \leq \sqrt{\log M} + \frac{2 \sqrt{2} \log M}{\Delta}
  \leq \frac{4 \log M}{\Delta}
  \label{eq:hedge-stochastic-proof-1}
\end{equation}
(since $\log M \geq 1$ as $M \geq 3$, $\Delta \leq 1$ and $2 \sqrt{2} \leq 3$),
which establishes~\eqref{eq:regret-stochastic-exp} for $T \leq t_0$.
In order to prove~\eqref{eq:regret-stochastic-exp} for $T \geq t_0 +1$, we start by decomposing the regret with respect to $i^*$ as
\begin{equation}
\label{eq:hedge-stochastic-proof-2}
R_{i^*,T}
= \wh L_{T} - L_{i^*, T}
= \wh L_{t_0} - L_{i^*, t_0} + \sum_{t=t_0+1}^T (\wh \ell_t - \ell_{i^*, t})
\, .
\end{equation}
Since $\wh L_{t_0} - L_{i^*, t_0} \leq R_{t_0}$ is controlled by~\eqref{eq:hedge-stochastic-proof-1}, it remains to upper bound the second term in~\eqref{eq:hedge-stochastic-proof-2}.
First, for every $t \geq t_0 +1$,
\begin{equation}
  \label{eq:instant-regret}
  \wh \ell_t - \ell_{i^*,t}
  = \sum_{i \neq i^*} v_{i,t} (\ell_{i,t} - \ell_{i^*,t})
  \, .
\end{equation}
Since $\bm \ell_t$ is independent of $\bm v_t$ (which is $\sigma (\bm \ell_1, \dots, \bm \ell_{t-1})$-measurable),
taking the expectation in~\eqref{eq:instant-regret} yields, denoting $\Delta_i = \E [\ell_{i,t} - \ell_{i^*,t}]$,
\begin{equation}
  \label{eq:instant-expected-regret}
  \E [\wh \ell_t  - \ell_{i^*,t}]
  = \sum_{i \neq i^*} \Delta_i \E [v_{i,t}]
  \, .
\end{equation}
First, for every $i \neq i^*$,
applying Hoeffding's inequality to the \iid centered variables $Z_{i,t} := - \ell_{i,t} + \ell_{i^*,t} + \Delta_i$, which belong to $[-1 + \Delta_i, 1 + \Delta_i]$, yields
\begin{align}
  \label{eq:stochastic-proof-hoeffding}
  \P \left( L_{i,t-1} - L_{i^*,t-1} < \frac{\Delta_i (t-1)}{2} \right)
  &= \P \left( \sum_{s=1}^{t-1} Z_{i,s} > \frac{\Delta_i (t-1)}{2} \right) \nonumber \\
  &\leq e^{ - \frac{t-1}{2} (\Delta_i/2)^2} \nonumber \\
  &= e^{ - {(t-1) \Delta_i^2}/{8}}
    \, .
\end{align}
On the other hand, if $L_{i,t-1} - L_{i^*,t-1} \geq \Delta_i (t-1) /2$, then 
\begin{align}
  \label{eq:hedge-stochastic-proof-3}
  v_{i,t}
  &= \frac{e^{- \eta_t (L_{i,t-1} - L_{i^*,t-1})}}{1 + \sum_{j \neq i^*} e^{- \eta_t (L_{j,t-1} - L_{i^*,t-1})}} \nonumber \\
  &\leq e^{- 2 \sqrt{(\log M) / t} \times \Delta_i (t-1) / 2} \nonumber \\
  &\leq e^{- \Delta_i \sqrt{(t-1) (\log M)/2}}
\end{align}
since $t \leq 2(t-1)$.
It follows from~\eqref{eq:hedge-stochastic-proof-3} and~\eqref{eq:stochastic-proof-hoeffding} that, for $t \geq t_0+1 \geq 2$,
\begin{align}
  \label{eq:hedge-stochastic-two-terms}
  \E [v_{i,t}]
  &\leq \P \left( L_{i,t-1} - L_{i^*,t-1} > \frac{\Delta_i (t-1)}{2} \right) + e^{- \Delta_i \sqrt{(t-1)(\log M)/2}} \nonumber \\
  &\leq e^{- (t-1) \Delta_i^2 / 8} + e^{- \Delta_i \sqrt{(t-1)(\log M)/2}}
   \, .
\end{align}
Now, a simple analysis of functions shows that the functions $f_1 (u) = u e^{-u}$ and $f_2 (u) = u e^{-u^2/2}$ are decreasing on $[1, + \infty)$.
Since $\Delta_i \geq \Delta$, this entails that
\begin{equation}
  \label{eq:hedge-stochastic-proof-5}
  \Delta_i e^{- (t-1) \Delta_i^2 / 8}
  = \frac{2}{\sqrt{t-1}} f_2 \left( \frac{\sqrt{t-1} \Delta_i}{2} \right)
  \leq \frac{2}{\sqrt{t-1}} f_2 \left( \frac{\sqrt{t-1} \Delta}{2} \right)
  = \Delta e^{- (t-1) \Delta^2 / 8}
\end{equation}
provided that $\frac{\sqrt{t-1} \Delta}{2} \geq 1$, \ie $t \geq 1 + \frac{4}{\Delta^2}$, which is the case since $t \geq t_0 +1 \geq 1 + \frac{8 \log M}{\Delta^2}$.
Likewise,
\begin{equation}
  \label{eq:hedge-stochastic-proof-6}
  \Delta_i e^{- \Delta_i \sqrt{(t-1) (\log M) /2}}
  \leq \Delta e^{- \Delta \sqrt{(t-1) (\log M) /2}}
\end{equation}
if $\Delta \sqrt{(t-1) (\log M) /2} \geq 1$, \ie $t \geq 1+ \frac{2}{(\log M) \Delta^2}$, which is ensured by $t \geq t_0 +1$.
It follows from~\eqref{eq:instant-expected-regret}, \eqref{eq:hedge-stochastic-two-terms}, \eqref{eq:hedge-stochastic-proof-5} and~\eqref{eq:hedge-stochastic-proof-6} that for every $t \geq t_0 + 1$:
\begin{align}
  \label{eq:hedge-stochastic-proof-simult}
  \E [\wh \ell_t - \ell_{i^*,t}]
  &\leq M \Delta e^{- (t-1) \Delta^2 / 8} + M \Delta e^{- \Delta \sqrt{(t-1) (\log M) /2}} \nonumber \\
  &= \big( M e^{- t_0 \Delta^2 / 8} \big) \big( \Delta e^{- (t-t_0- 1) \Delta^2 / 8} \big) + \big( M e^{- \Delta \sqrt{(t-1) (\log M) /8}} \big) \big( \Delta e^{- \Delta \sqrt{(t-1) (\log M) /8}} \big) \nonumber \\
  &\leq \Delta e^{- (t-t_0- 1) \Delta^2 / 8} + \Delta e^{- \Delta \sqrt{(t-1)/8}}
\end{align}
where inequality~\eqref{eq:hedge-stochastic-proof-simult} comes from the bound $M e^{-t_0 \Delta^2 /8} \leq 1$ (since $t_0 \geq \frac{8 \log M}{\Delta^2}$) and from the fact that $M e^{- \Delta \sqrt{(t-1) (\log M) /8}} \leq 1$ amounts to $t \geq 1 + \frac{8 \log M}{\Delta^2}$, that is, to $t \geq t_0 +1$.
Summing inequality~\eqref{eq:hedge-stochastic-proof-simult} yields, for every $T \geq t_0+1$,
\begin{align}  
  \E [ \sum_{t = t_0 +1}^T (\ell_{t} - \ell_{i^*,t}) ]
  &\leq \sum_{t=t_0+1}^T \left\{ \Delta e^{- (t-t_0- 1) \Delta^2 / 8} + \Delta e^{- \Delta \sqrt{(t-1)/8}} \right\}  \nonumber \\
  &\leq \Delta \sum_{t \geq 0} e^{-t \Delta^2/8} + \Delta \sum_{t \geq 1} e^{- (\Delta / \sqrt{8}) \sqrt{t}} \nonumber \\
  &\leq \Delta \left( 1 + \frac{8}{\Delta^2} \right) + \Delta \times \frac{2}{(\Delta / \sqrt{8})^2} \label{eq:hedge-stochastic-proof-7} \\
  &
  \leq \frac{25}{\Delta}
\end{align}
where inequality~\eqref{eq:hedge-stochastic-proof-7} comes from Lemma~\ref{lem:tail} below.
Finally, combining inequalities~\eqref{eq:hedge-stochastic-proof-1} and~\eqref{eq:hedge-stochastic-proof-7} yields the pseudo-regret bound $\pseudoregret_T \leq \frac{4 \log M + 25}{\Delta}$.

\begin{lemma}
  \label{lem:tail}
  For every $\alpha > 0$,
  \begin{align}
    \label{eq:tail-bound}
    \sum_{t \geq 1} e^{-\alpha t} &\leq \frac{1}{\alpha} \\
    \label{eq:tail-bound-sqrt}
    \sum_{t \geq 1} e^{- \alpha \sqrt{t}} &\leq \frac{2}{\alpha^2} \, .
  \end{align}
\end{lemma}

\begin{proof}
  Since the functions $t \mapsto e^{-\alpha {t}}$ and $t \mapsto e^{-\alpha \sqrt{t}}$ are decreasing on $\R^+$, we have
  \begin{align*}
    &\sum_{t \geq 1} e^{-\alpha t}
    \leq \int_0^{\infty} e^{-\alpha t} \di t
    = \frac{1}{\alpha}
    \, 
    \\
  &\sum_{t \geq 1} e^{-\alpha \sqrt{t}}
  \leq \int_{0}^{+\infty} e^{-\alpha \sqrt{t}} \di t
  \underset{u = \alpha \sqrt t}{=} \frac{2}{\alpha^2}
  \int_{0}^{+\infty} u e^{-u} \di u
  = \frac{2}{\alpha^2}.
  \end{align*}
\end{proof}

\begin{remark}
  \label{rem:pseudoregret}
  While the upper bound of Theorem~\ref{thm:hedge-stochastic} is stated for the pseudo-regret $\pseudoregret_T$, a similar upper bound holds for the expected regret $\E [R_T]$.
  Indeed, under the assumptions of Theorem~\ref{thm:hedge-stochastic}, for every $T \geq \frac{4 \log M}{\Delta^2}$, we have
  $\E [R_T] \leq \pseudoregret_T + \frac{1.1}{\Delta}$.
\end{remark}

\begin{proof}
  Note that $\E [R_T] - \pseudoregret_T = \E [L_{i^*,T} - \min_{1\leq i \leq T} L_{i,T}]$.
  For every $a \geq 0$, Hoeffding's inequality (applied to the \iid centered variables $\ell_{i^*,t} - \ell_{i,t} + \Delta_i \in [-1 + \Delta_i, 1 + \Delta_i]$, $1\leq t \leq T$) entails 
  \begin{align}
    \label{eq:proof-pseudo-regret-1}
    \P \left( L_{i^*,T} - \min_{1\leq i \leq T} L_{i,T} \geq a \right)
    &\leq \sum_{i \neq i^*} \P \left( L_{i^*,T} - L_{i,T} + \Delta_i T \geq \Delta_i T + a \right) \nonumber \\
    &\leq \sum_{i \neq i^*} e^{- (\Delta_i T + a)^2 / (2T)} \\
    &\leq M e^{- T \Delta^2/2} e^{-a^2/(2T)} \nonumber \\
    &\leq e^{-T \Delta^2/4} e^{-a^2/(2T)}
      \label{eq:proof-pseudo-regret-2}
      \, ,
  \end{align}
  where inequality~\eqref{eq:proof-pseudo-regret-2} comes from the fact that $M e^{-T \Delta^2/4} \leq 1$ since $T \geq \frac{4 \log M}{\Delta^2}$.
  Since the random variable $L_{i^*,T} - \min_{1\leq i \leq T} L_{i,T}$ is nonnegative, this implies that
  \begin{align}    
    \E \left[ L_{i^*,T} - \min_{1\leq i \leq T} L_{i,T} \right]
    &= \int_0^{\infty} \P \left( L_{i^*,T} - \min_{1\leq i \leq T} L_{i,T} \geq a \right) \di a \nonumber \\
    &\leq e^{-T \Delta^2/4} \int_0^{\infty} e^{-a^2/(2T)} \di a \nonumber \\
    &= \sqrt{\frac{\pi}{2}} \cdot \sqrt{T} e^{-T \Delta^2/4} \nonumber \\
    &= \frac{\sqrt{\pi}}{\Delta} \big[ \Delta \sqrt{T/2} \cdot e^{- (\Delta \sqrt{T/2})^2/2} \big] \nonumber \\
    &\leq \frac{\sqrt{\pi/e}}{\Delta}
      \label{eq:proof-pseudo-regret-3}
  \end{align}
  where inequality~\eqref{eq:proof-pseudo-regret-3} comes from the fact that the function $u \mapsto u e^{-u^2/2}$ attains its maximum on $\R^+$ at $u=1$.
  This concludes the proof, since $\sqrt{\pi/e}\leq 1.1$.
\end{proof}

\subsection{Proof of Proposition~\ref{prop:lowerbound-gap}}
\label{sec:proof-lowerbound-gap}

Fix $M$, $\Delta$ and $T$ as in Proposition~\ref{prop:lowerbound-gap}.
For $i^* \in \{ 1, \dots, M\}$, denote $\P_{i^*}$ the following distribution on $[0, 1]^{M \times T}$: if $(\ell_{i,t})_{1\leq i \leq M, 1\leq t \leq T} \sim \P_{i^*}$, then the variables $\ell_{i,t}$ are independent Bernoulli variables, of parameter $\frac{1}{2} - \Delta$ if $i = i^*$ and $\frac{1}{2}$ otherwise; also, denote by $\E_{i^*}$ the expectation with respect to $\P_{i^*}$.
Let $\A = (A_t)_{1 \leq t \leq T}$ be any Hedging algorithm, where $A_t: [0, 1]^{M \times (t-1)} \to \probas_M$ maps past losses $(\bm \ell_{1}, \dots, \bm \ell_{t-1})$ to an element of the probability simplex $\probas_M \subset \R^M$ on $\{ 1, \dots, M \}$.
For any $i^* \in \{ 1, \dots, M \}$, let $\pseudoregret_{T} (i^*, \A)$ denote the pseudo-regret of algorithm $\A$ under the distribution $\P_{i^*}$.
Since $\bm \ell_t$ is independent of $\bm v_t$ under $\P_{i^*}$, we have
\begin{equation}
  \label{eq:lowerbound-gap-pseudoregret}
  \pseudoregret_T (i^*, \A)
  = \sum_{t=1}^T \sum_{i \neq i^*} \E_{i^*} \big[ v_{i,t} (\ell_{i,t} - \ell_{i^*,t}) \big]
  = \Delta \sum_{t=1}^T \sum_{i \neq i^*} \E_{i^*} [v_{i,t}]
  = \Delta \sum_{t=1}^T \E_{i^*} [1 - v_{i^*,t}]
\end{equation}
with $\bm v_t := A_t (\bm\ell_1, \dots, \bm\ell_{t-1})$.
It follows from Equation~\eqref{eq:lowerbound-gap-pseudoregret} that, for every $\A$ and $i^*$, $\pseudoregret_T (i^*, \A)$ increases with $T$.
Hence, without loss of generality we may assume that {$T = \lfloor (\log M) / (16 \Delta^2) \rfloor$}.
The maximum pseudo-regret of $\A$ on the instances $\P_{i^*}$ is lower-bounded as follows:
\begin{equation}
  \label{eq:lowerbound-gap-1}
  \sup_{1 \leq i^* \leq M} \pseudoregret_T (i^*, \A)
  \geq \frac{1}{M} \sum_{1 \leq i^* \leq M} \pseudoregret_T (i^*, \A)
  = \frac{1}{M} \sum_{1 \leq i^* \leq M} \Delta \sum_{t=1}^T \E_{i^*} [ 1 - v_{i^*,t}]
  \, .
\end{equation}

We now ``randomize'' the algorithm $\A$, by replacing it with a randomized algorithm which picks expert $i$ at time $t$ with probability $v_{i,t}$.
Formally, let $\wt P = \uniformdist ([0, 1])^{\otimes T}$ be the distribution of $T$ independent uniform random variables on $[0, 1]$, and denote $\wt \P_{i^*} = \P_{i^*} \otimes \wt P$ for $i^* \in \{ 1, \dots, M \}$.
Furthermore, for every $\bm v \in \probas_M$, let $I_{\bm v} : [0, 1] \to \{ 1, \dots, M \}$ be a measurable map such that $\P (I_{\bm v} (U) = i) = v_i$ for every $i \in \{ 1, \dots, M \}$, where $U \sim \uniformdist ([0 ,1])$.
For every sequence of losses $\bm \ell_1, \dots, \bm \ell_T$ and random variables $U_1, \dots, U_T$ and every $1 \leq t \leq T$, let $I_t = I_{\bm v_t} (U_t)$, where $\bm v_t = A_t (\bm \ell_1, \dots, \bm \ell_t)$.

Denote by $\wt \E_{i^*}$ the expectation with respect to $\wt \P_{i^*}$.
By definition of $I_{\bm v}$, we have $\E_{i^*} [ v_{i^*, t} ] = \wt \E_{i^*} [ \indic{I_t = i^*} ]$ so that, denoting $N_i = \sum_{i=1}^T \indic{I_t = i}$ the number of times expert $i$ is picked,
\begin{equation*}
  \sum_{t=1}^T \E_{i^*} [1 - v_{i^*,t}] = \wt \E_{i^*} [ T - N_{i^*} ]
  \geq \P_{i^*} (N_{i^*} \leq T/2) \cdot \frac{T}{2}
  \, .
\end{equation*}
Hence, letting $A_i \subseteq [0, 1]^{M \times T} \times [0, 1]^T$ be the event $\{ N_i > T/2 \}$,
Equation~\eqref{eq:lowerbound-gap-1} implies that
\begin{equation}
  \label{eq:lowerbound-gap-2}
  \sup_{1 \leq i^* \leq M} \pseudoregret_T (i^*, \A)
  \geq \frac{\Delta T}{2} \times \frac{1}{M} \sum_{1\leq i^* \leq M} \big( 1 - \wt \P_{i^*} (A_{i^*}) \big)
  \, .
\end{equation}

It now remains to upper bound $\frac{1}{M} \sum_{i^*} \wt \P_{i^*} (A_{i^*})$.
To do this, first note that the events $A_{i^*}$, $1 \leq i^* \leq M$, are pairwise disjoint.
Hence, Fano's inequality \citep[see][p.2]{gerchinovitz2017fano} implies that, for every distribution $\wt \Qbb$ on $[0, 1]^{M \times T} \times [0, 1]^T$,
\begin{equation}
  \label{eq:lowerbound-gap-fano}
  \frac{1}{M} \sum_{1\leq i^* \leq M} \wt \P_{i^*} ( A_{i^*})
  \leq \frac{1}{\log M} \bigg\{ \frac{1}{M} \sum_{1 \leq i^* \leq M} \kll{\wt \P_{i^*}}{\wt \Qbb} + \log 2 \bigg\}
\end{equation}
where $\kll{\P}{\Qbb}$ denotes the Kullback-Leibler divergence between $\P$ and $\Qbb$.
Here, we take $\wt \Qbb = \Qbb \otimes \wt P$, where $\Qbb$ is the product of Bernoulli distributions $\bernoullidist (1/2)^{\otimes T}$.
This choice leads to
\begin{equation*}
  \kll{\wt \P_{i^*}}{\wt \Qbb}
  = \kll{\P_{i^*}}{\Qbb}
  = T \cdot \kll{\bernoullidist(1/2 - \Delta)}{\bernoullidist(1/2)}
  \leq 4 T \Delta^2
  \leq \frac{\log M}{4}
  \, ,
\end{equation*}
where the first bound is obtained by comparing KL and $\chi^2$ divergences \citep[Lemma~2.7]{tsybakov2009nonparametric}.
Hence, inequality~\eqref{eq:lowerbound-gap-fano} becomes (recalling that $M \geq 4$)
\begin{equation*}
  \frac{1}{M} \sum_{1\leq i^* \leq M} \wt \P_{i^*} ( A_{i^*})
  \leq \frac{(\log M)/4}{\log M} + \frac{\log 2}{\log M}
  \leq \frac{3}{4}
  \, ;
\end{equation*}
plugging this into~\eqref{eq:lowerbound-gap-2} yields, noting that $T = \lfloor (\log M) / (16 \Delta^2) \rfloor \geq (\log M) / (32 \Delta^2)$ since $(\log M)/(16 \Delta^2) \geq 1$ (as $M \geq 4$ and $\Delta \leq \frac{1}{4}$),
\begin{equation*}
  \sup_{1 \leq i^* \leq M} \pseudoregret_T (i^*, \A)
  \geq \frac{\Delta T}{2} \times \frac{1}{4}
  \geq \frac{\log M}{256 \Delta}
  \, .
\end{equation*}
This concludes the proof.

\subsection{Proof of Theorem~\ref{thm:adv-gap} and Corollary~\ref{cor:hedge-martingale}}
\label{sec:proof-thm-adv-gap}

Let $t_0 $ be the smallest integer $t\geq 1$ such that $M e^{- c_0 \Delta \sqrt{t \log (M) / 8}} \leq \Delta$, namely $t_0 = \left\lceil \frac{8}{c_0^2 \Delta^2} \frac{\log^2 (M / \Delta)}{\log M} \right\rceil$.
  Note that $\sqrt{t_0} \leq \sqrt{1 + \frac{8}{c_0^2 \Delta^2} \frac{\log^2 (M / \Delta)}{\log M}} \leq 1 + \frac{\sqrt{8}}{c_0 \Delta} \frac{\log (M / \Delta)}{\sqrt{\log M}}$.
  Let $t_1 := t_0 \vee \tau_0$.
  For every $T \leq t_1$, the regret bound in the assumption of Theorem~\ref{thm:adv-gap} implies
  \begin{align}
    \label{eq:proof-adv-1}
    R_T
    &\leq c_1 \sqrt{T \log M} \nonumber \\    
    &\leq c_1 \sqrt{\tau_0 \log M} + c_1 \sqrt{t_0 \log M} \nonumber \\    
    &\leq c_1 \sqrt{\tau_0 \log M} + c_1 \sqrt{\log M} + \frac{\sqrt{8} \log(M/\Delta)}{c_0 \Delta}
  \end{align}
  which implies~\eqref{eq:regret-gap-adv} with $c_2 = c_1 + \frac{\sqrt{8}}{c_0}$ and $c_3 = \frac{\sqrt{8}}{c_0}$ (since $1\leq \sqrt{\log M} \leq \frac{\log M}{\Delta}$).
  From now on, assume that $T \geq t_1 + 1$.
  Since $T \geq \tau_0$, we have $R_T = \wh L_T - L_{i^*,T}$, so that
  \begin{equation}
    \label{eq:proof-adv-2}
    R_T
    = \wh L_{t_1} - L_{i^*, t_1} + \sum_{t=t_1 + 1}^T \big( \wh \ell_t - \ell_{i^*, t} \big)
    \, .
  \end{equation}
  In addition, we have for $t\geq t_1 +1$
  \begin{align}    
    \wh \ell_t - \ell_{i^*, t}
    &= \sum_{i \neq i^*} v_{i,t} (\ell_{i,t} - \ell_{i^*, t}) \nonumber \\
    &\leq \sum_{i \neq i^*} v_{i,t} \nonumber \\
    &= \sum_{i\neq i^*} \frac{ e^{- \eta_t (L_{i,t-1} - L_{i^*, t-1})}}{1 + \sum_{j\neq i^*} e^{- \eta_t (L_{j,t-1} - L_{i^*, t-1})}} \nonumber \\
    &\leq \sum_{i\neq i^*} e^{- c_0 \sqrt{(\log M)/t} \times \Delta (t-1)} \label{eq:proof-adv-3} \\
    &\leq M e^{- c_0 \Delta \sqrt{(t-1)(\log M)/2}} \nonumber
    \\
    &\leq \big( M e^{- c_0 \Delta \sqrt{t_0 (\log M)/8}} \big) e^{- c_0 \Delta \sqrt{(t-1)/8}}  \label{eq:proof-adv-5} \\
    &\leq \Delta e^{- c_0 \Delta \sqrt{(t-1)/8}} \label{eq:proof-adv-6}
  \end{align}
  where~\eqref{eq:proof-adv-3} comes from the fact that $\eta_t \geq c_0 \sqrt{(\log M)/t}$ and $L_{i,t-1} - L_{i^*,t-1} \geq \Delta (t-1)$ (since $t-1\geq t_1 \geq \tau_0$),~\eqref{eq:proof-adv-5} from the fact that $t-1 \geq t_0$ and $\log M \geq 1$, and~\eqref{eq:proof-adv-6} from the fact that $M e^{- c_0 \Delta \sqrt{t_0 (\log M)/8}} \leq \Delta$.
  Summing inequality~\eqref{eq:proof-adv-6}, we obtain 
  \begin{align}        
    \sum_{t=t_1+1}^T (\wh \ell_t - \ell_{i^*, t}) \nonumber
    &\leq \sum_{t=t_1+1}^T \Delta e^{- c_0 \Delta \sqrt{(t-1)/8}} \\
    &\leq \Delta \sum_{t\geq 1} e^{- c_0 \Delta \sqrt{t/8}} \nonumber \\
    &\leq \Delta \times \frac{2}{(c_0 \Delta / \sqrt{8})^2} \label{eq:proof-adv-7} \\
    &= \frac{16}{c_0^2 \Delta}
      \label{eq:proof-adv-tail}
  \end{align}
  where~\eqref{eq:proof-adv-7} follows from Lemma~\ref{lem:tail}.
  Combining~\eqref{eq:proof-adv-2},~\eqref{eq:proof-adv-1} and~\eqref{eq:proof-adv-tail} proves Theorem~\ref{thm:adv-gap} with $c_2 = c_1 + \frac{\sqrt{8}}{c_0}$, $c_3 = \frac{\sqrt{8}}{c_0}$ and $c_4 = \frac{16}{c_0^2}$.

  \begin{proof}[Proof of Corollary~\ref{cor:hedge-martingale}]
    Define $\tau = \sup \{ t \geq 0 , \exists i \neq i^* , L_{i,t} - L_{i^*, t} \leq \frac{\Delta t}{2} \}$.
    By Lemma~\ref{lem:crossing-time} below, for every $\eps > 0$ we have, with probability at least $1-\eps$, $\tau \leq {8 ( \log M + \log \eps^{-1})}/{\Delta^2}$.
    By Theorem~\ref{thm:adv-gap}, this implies that, with probability at least $1-\eps$,
    \begin{align*}
      \label{eq:proof-hedge-martingale-1}
      R_T
      &\leq c_1 \sqrt{\tau \log M} + \frac{c_2 \log M + c_3 \log {\Delta}^{-1} + c_4}{\Delta / 2} \\
      &\leq \left( c_1 \sqrt{8} + 2 c_2 \right) \frac{\log M}{\Delta} + c_1 \frac{\sqrt{8\log M \log \eps^{-1}}}{\Delta} + 2 c_3 \frac{\log \Delta^{-1}}{\Delta} + \frac{2 c_4}{\Delta}        
    \end{align*}
    where $c_2, c_3, c_4$ are the constants of Theorem~\ref{thm:adv-gap}.
    The bound~\eqref{eq:regret-martingale-exp} on the pseudo-regret is obtained similarly from Theorem~\ref{thm:adv-gap}, by using the fact that $\pseudoregret_T \leq \E [R_T]$ and
    \begin{equation*}
      \E [ \sqrt{\tau \log M}] \leq \sqrt{\E [\tau] \log M}
      \leq \sqrt{\log M} \sqrt{ 1 + \frac{8 (\log M + 1)}{\Delta^2}}
      \leq \sqrt{\log M} \Big( 1 + \frac{\sqrt{8 \log M} + 1}{\Delta} \Big)
    \end{equation*}
    which is smaller than $(2 + \sqrt{8}) ({\log M})/{\Delta} \leq 5 (\log M)/\Delta$ since $M \geq 3$ and $\Delta \leq 1$.
  \end{proof}

\begin{lemma}
  \label{lem:crossing-time}
  Let $(\ell_{i,t})_{1\leq i \leq M, t\geq 1}$ be as in Theorem~\ref{thm:hedge-stochastic}.
  Denote $\tau = \sup \{ t \geq 0 , \exists i \neq i^* , L_{i,t} - L_{i^*, t} \leq \frac{\Delta t}{2} \}$.
  We have
  \begin{equation}
    \label{eq:crossing-time-exp}
    \E [ \tau ]    
    \leq 1 + \frac{8 (\log M + 1)}{\Delta^2}    
    \, ,    
  \end{equation}
  and for every $\eps \in (0, 1)$,
\begin{equation}
  \label{eq:crossing-time-prob}
  \P \Big( \tau \geq \frac{8 (\log M + \log \eps^{-1})}{\Delta^2} \Big) \leq \eps \, .
\end{equation}    
\end{lemma}

\begin{proof}[Proof of Lemma~\ref{lem:crossing-time}]
  For every $i \neq i^*$ and $t \geq 1$, let $\Delta_{i,t} := \E [ \ell_{i,t} - \ell_{i^*,t} \cond \F_{t-1}]$.
  Using the Hoeffding-Azuma's maximal inequality to the $(\F_t)_{t\geq 1}$-martingale difference sequence $Z_{i,t} = - (L_{i,t} - L_{i^*,t}) + \Delta_{i,t}$ (such that $\Delta_{i,t} - 1 \leq Z_{i,t} \leq \Delta_{i,t} +1$), together with the fact that $\Delta_{i,t} \geq \Delta$, implies that
  \begin{equation}
    \label{eq:proof-crossing-1}
    \P \left( \exists t \geq t_0, L_{i,t} - L_{i^*,t} \leq \frac{\Delta t}{2} \right)
    \leq \P \left( \sup_{t \geq t_0} \frac{1}{t} \left( \sum_{s=1}^t Z_{i,s} \right) \geq \frac{\Delta}{2} \right)
    \leq e^{- t_0 \Delta^2/8}
    \, .
  \end{equation}
  By a union bound, equation~\eqref{eq:proof-crossing-1} implies that
  \begin{equation}
    \label{eq:proof-crossing-2}
    \P \left( \tau \geq t_0 \right)
    \leq M e^{-t_0 \Delta^2/8}
    \, .
  \end{equation}
  Solving for the probability level in~\eqref{eq:proof-crossing-2} yields the high probability bound~\eqref{eq:crossing-time-prob} on $\tau$.
  The bound on $\tau$ in expectation~\eqref{eq:crossing-time-exp} ensues by integrating the high-probability bound over $\eps$.
\end{proof}

We recall Hoeffding-Azuma's maximal inequality for bounded martingale difference sequences \citep{hoeffding1963probability,azuma1967weighted}.
While it follows from a standard argument, we provide a short proof for completeness, since the inequality given in Proposition~\ref{lem:hoeffding-maximal} below differs slightly from the one given in~\citet{hoeffding1963probability}.

\begin{proposition}[Hoeffding-Azuma's maximal inequality]
  \label{lem:hoeffding-maximal}
  Let $(Z_t)_{t \geq 1}$ be a sequence of random variables adapted to a filtration $(\F_t)_{t\geq 1}$.
  Assume that $Z_t$ is a martingale difference sequence: $\E [Z_t \cond \F_{t-1}] = 0$ for any $t \geq 1$, and that $A_t - 1 \leq Z_t \leq A_t+1$ almost surely, where $A_t$ is $\F_{t-1}$-measurable.
  Then, denoting $S_n := \sum_{t=1}^n Z_t$, we have for every $n \geq 1$ and $a \geq 0$:
  \begin{equation}
    \label{eq:hoeffding-maximal}
    \P \left( \sup_{m \geq n} \frac{S_m}{m} \geq a \right)
    \leq e^{- n a^2/2}
    \, .
  \end{equation}
\end{proposition}

\begin{proof}
  Fix $\lambda > 0$.
  By Hoeffding's inequality, $\E [e^{\lambda Z_t} \cond \F_{t-1}] \leq e^{\lambda^2/2}$, so that the sequence $M_t^{\lambda} := \exp \big( \lambda S_t - \lambda^2 t / 2 \big)$ is a positive supermartingale.
  Hence, Doob's supermartingale inequality implies that for $\eps \in (0, 1]$:
  \begin{equation}
    \label{eq:proof-hoeffding-maximal-1}
    \P \Big( \sup_{t \geq 1} M_t^{\lambda} \geq \frac{1}{\eps} \Big)
    \leq \frac{\E [M_0^\lambda]}{1/\eps} = \eps
    \, .
  \end{equation}
  Rearranging~\eqref{eq:proof-hoeffding-maximal-1} and letting $\lambda = \sqrt{2 \log (1/\eps) / n}$ yields: with probability $1 - \eps$, for every $t \geq n$,
  \begin{equation}
    \label{eq:proof-hoeffding-maximal-2}
    \frac{S_t}{t}
    \leq \frac{\log \left( {1}/{\eps} \right)}{\lambda t} + \frac{\lambda}{2}
    = \sqrt{\frac{\log (1/\eps)}{2}} \left( \frac{\sqrt{n}}{t} + \frac{1}{\sqrt{t}} \right)
    \leq \sqrt{\frac{2 \log (1/\eps)}{n}}
    \, .
  \end{equation}
  Setting $\eps = e^{- n a^2/2}$ in~\eqref{eq:proof-hoeffding-maximal-2} gives the desired bound.
\end{proof}

\subsection{Proof of Proposition~\ref{prop:lower-bound-hedge-cst}}
\label{sec:proof-lower-bound-hedge-cst}

Note that, since the loss vectors $\bm \ell_t$ are in fact deterministic, $\pseudoregret_T = R_T$.
Denoting $(v_{i,t})_{1 \leq i \leq M}$ the weights selected by the Constant Hedge algorithm at time $t$, and letting $c = c_0 \sqrt{\log M}$, we have
\begin{align}
  R_T
  &= \sum_{t=1}^T \sum_{i=2}^M v_{i, t} (\ell_{i, t} - \ell_{1, t}) \nonumber \\
  &= \sum_{t=1}^T \sum_{i=2}^M \frac{\exp \big( - \frac{c}{\sqrt{T}} (L_{i,t-1}-L_{1,t-1}) \big)}{1 + \sum_{2 \leq i' \leq M} \exp \big( - \frac{c}{\sqrt{T}} (L_{i',t-1}-L_{1,t-1}) \big)} \nonumber \\
  &= \sum_{t=1}^T \frac{ (M-1) \exp \big( - \frac{c}{\sqrt{T}} (t-1) \big)}{1 + (M-1) \exp \big( - \frac{c}{\sqrt{T}} (t-1) \big)}    \label{eq:proof-lower-cst1}
    \, .
\end{align}
Now, let $t_0 \geq 0$ be the largest integer such that $(M-1) \exp ( - \frac{c}{\sqrt{T}} t ) \geq 1/2$, namely
\begin{equation*}
    t_0 = \Big\lfloor \frac{\sqrt{T}}{c} \log (2(M-1)) \Big\rfloor.
\end{equation*}
It follows from Equation~\eqref{eq:proof-lower-cst1} that
\begin{equation}
  \label{eq:proof-lower-cst2}
  R_T
  \geq \sum_{t=1}^{T \wedge (t_0 + 1)} \frac{ (M-1) \exp \big( - \frac{c}{\sqrt{T}} (t-1) \big)}{1 + (M-1) \exp \big( - \frac{c}{\sqrt{T}} (t-1) \big)} 
  \geq \frac{1}{3} \min (T, t_0 + 1)
\end{equation}
where the second inequality comes from the fact that $\frac{x}{1 + x} \geq \frac{1}{3}$ for $x \geq \frac{1}{2}$, 
which we apply to $x = (M - 1) \exp ( - \frac{c}{\sqrt{T}} (t-1)) \geq \frac{1}{2}$ for $t \leq T \wedge (t_0+1) \leq t_0 +1$.
In order to establish inequality~\eqref{eq:lower-bound-hedge-cst}, it remains to note that
\begin{equation*}
  t_0 + 1
  \geq \frac{\sqrt{T}}{c}
  \log \big( 2 (M-1) \big)
  \geq \frac{\sqrt{T \log M}}{c_0}
  \, ,
\end{equation*}
since $2(M-1) \geq M$ and $c = \sqrt{c_0 \log M}$.

Now, consider the Hedge algorithm with doubling trick. Assume that $T \geq 2$, and let $k \geq 1$ such that $T_k \leq T < T_{k+1}$.
Since $R_T = \sum_{t=1}^T \sum_{2\leq i \leq M} v_{i, t} (\ell_{i, t} - \ell_{1, t})$ and each of the terms in the sum is nonnegative, $R_T$ is lower bounded by the cumulative regret on the period $\iint{T_{k-1}}{T_k - 1}$.
During this period of length $T_{k-1}$, the algorithm reduces to the Hedge algorithm with constant learning rate $c_0 \sqrt{\log (M) / T_{k-1}}$, so that the above bound~\eqref{eq:lower-bound-hedge-cst} applies; further bounding $T_{k-1} \geq \frac{T}{4}$ establishes~\eqref{eq:lower-bound-hedge-doubling}.

\subsection{Proof of Proposition~\ref{prop:second-order-bernstein}}
\label{sec:proof-second-order-bernstein}

By convexity of $x \mapsto x^2$ and concavity of $x \mapsto x^{\beta}$,
we have:
\begin{align}
  \label{eq:proof-second-bernstein-1}
  \E [ (\wh \ell_{t} - \ell_{i^*,t})^2 ]
  &\leq \E \bigg[ \sum_{i=1}^M v_{i,t} (\ell_{i,t} - \ell_{i^*,t})^2 \bigg] \\
  &= \E \bigg[ \sum_{i=1}^M v_{i,t} \E \left[ (\ell_{i,t} - \ell_{i^*,t})^2 \cond \F_{t-1} \right] \bigg] \nonumber \\
  &\leq B \E \bigg[ \sum_{i=1}^M v_{i,t} \E \left[ \ell_{i,t} - \ell_{i^*,t} \cond \F_{t-1} \right]^\beta \bigg] \label{eq:proof-second-bernstein-2} \\
  &\leq B \E \bigg[ \sum_{i=1}^M v_{i,t} \E \left[ \ell_{i,t} - \ell_{i^*,t} \cond \F_{t-1} \right] \bigg]^\beta \label{eq:proof-second-bernstein-3} \\
  &= B \E [\wh \ell_t - \ell_{i^*,t}]^\beta
\end{align}
where inequalities~\eqref{eq:proof-second-bernstein-1} and~\eqref{eq:proof-second-bernstein-3} come from Jensen's inequality, and~\eqref{eq:proof-second-bernstein-2} from the Bernstein condition~\eqref{eq:bernstein-condition}.
Taking the expectation of the regret bound~\eqref{eq:second-order-regret}, we obtain
\begin{align}
  \label{eq:proof-second-1}
  \E [R_{i^*,T}]
  &\leq \E \Bigg[ C_1 \sqrt{(\log M) \sum_{t=1}^T (\wh \ell_t - \ell_{i^*,t})^2} + C_2 \log M \Bigg] \nonumber \\
  &\leq C_1 \sqrt{(\log M) \sum_{t=1}^T \E \big[ (\wh \ell_t - \ell_{i^*,t})^2 \big]} + C_2 \log M \\
  &\leq C_1 \sqrt{(\log M) B \sum_{t=1}^T \E \big[ \wh \ell_t - \ell_{i^*,t} \big]^\beta} + C_2 \log M \nonumber \\
  &= C_1 \sqrt{B T \log M} \bigg( \frac{1}{T} \sum_{t=1}^T \E \big[ \wh \ell_t - \ell_{i^*,t} \big]^\beta \bigg)^{1/2} + C_2 \log M \nonumber \\
  &\leq C_1 \sqrt{B T \log M} \left( \frac{\E [R_{i^*,T}]}{T} \right)^{\beta/2} + C_2 \log M
    \label{eq:proof-second-ineq}
\end{align}
where inequalities~\eqref{eq:proof-second-1} and~\eqref{eq:proof-second-ineq} come from Jensen's inequality.
Letting $r = {\E [R_{i^*,T}]}/{T}$ and $u = ({\log M})/{T}$,
inequality~\eqref{eq:proof-second-ineq} writes $r \leq C_1 \sqrt{B u} r^{\beta/2} + C_2 u$.
This implies that (depending on which of these two terms is larger) either $r \leq 2 C_2 u$, or $r \leq 2 C_1 \sqrt{B u} r^{\beta/2}$, and the latter condition amounts to $r \leq (2 C_1)^{2/(2-\beta)} (B u)^{1/(2-\beta)}$.
This entails that
\begin{equation*}
  r \leq (2 C_1)^{\frac{2}{2-\beta}} (B u)^{\frac{1}{2-\beta}} + 2 C_2 u
  \, ,
\end{equation*}
which amounts to
\begin{equation}
  \label{eq:proof-second-final}
  \E [R_{i^*,T}]
  \leq C_3 (B \log M)^{\frac{1}{2-\beta}} T^{\frac{1-\beta}{2-\beta}} + C_4 \log M
\end{equation}
where $C_3 = (2C_1)^{2/(2-\beta)} \leq \max (1, 4C_1^2)$ and $C_4 = 2C_2$.

\subsection{Proof of Theorem~\ref{thm:hedge-no-bernstein}}
\label{sec:proof-lowerbound-no-bernstein}

Consider the constant losses $\ell_{1, t} = 0$, $\ell_{i,t} = \Delta$ where $\Delta = {1} \wedge c_0^{-1} \sqrt{(\log M)/T}$.
These losses satisfy the $(1, 1)$-Bernstein condition since, for every $i>1$,
$\E [ (\ell_{i,t} - \ell_{1,t})^2 ] = \Delta^2 \leq \Delta = \E [\ell_{i,t} - \ell_{1,t}]$.
On the other hand, the regret of the Hedge algorithm with learning rate $\eta_t = c_0 \sqrt{(\log M)/t}$ writes
\begin{align}
  \pseudoregret_T
  &= \sum_{t=1}^T \sum_{i\neq 1} \E [v_{i,t} (\ell_{i,t} - \ell_{1,t})] \nonumber \\
  &= \Delta \sum_{t=1}^T \frac{(M-1) e^{-\eta_t \Delta (t-1)}}{1 + (M-1) e^{-\eta_t \Delta (t-1)}} \nonumber \\
  &\geq \frac{\Delta}{3} \sum_{t=1}^T \bm 1 \Big( (M-1) e^{-\eta_t \Delta (t-1)} \geq \frac{1}{2} \Big) \nonumber \\
  &\geq \frac{\Delta}{3} \sum_{t=1}^T \bm 1 \left( M e^{- c_0 \Delta \sqrt{(t-1) \log M}} \geq 1 \right) \label{eq:proof-hedge-no-bernstein-0} \\
  &\geq \frac{\Delta}{3} \times \min \left( \frac{\log M}{c_0^2 \Delta^2}, T \right) \nonumber \\
  &= \frac{1}{3} \min \Big( \frac{1}{c_0} \sqrt{T \log M}, {T} \Big)
    \, , \label{eq:proof-hedge-no-bernstein-0b}
\end{align}
where~\eqref{eq:proof-hedge-no-bernstein-0} relies on the inequalities $2(M-1) \geq M$ and $({t-1})/{\sqrt{t}} \leq \sqrt{t-1}$ for $M \geq 2, t\geq 1$, while~\eqref{eq:proof-hedge-no-bernstein-0b}
is obtained by noting that ${ (\log M)}/({c_0^2 \Delta^2}) \geq T$ since $\Delta \leq c_0^{-1} \sqrt{(\log M)/T}$ and substituting for $\Delta$.

\subsection{Proof of Theorem~\ref{thm:hedge-characterize-gap}}
\label{sec:proof-hedge}

Assume that the loss vectors $\bm \ell_1, \bm \ell_2, \dots$ are \iid, and denote $i^* = \argmin_{1\leq i \leq M} \E [\ell_{i,t}]$ (which is assumed to be unique), $\Delta = \min_{i \neq i^*} \Delta_i > 0$ where $\Delta_i = \E [\ell_{i,t} - \ell_{i^*,t}]$ and $j \in \{1, \dots, M\}$ such that $\Delta_j = \Delta$.
The Decreasing Hedge algorithm with learning rate $\eta_t = c_0 \sqrt{(\log M)/t}$ satisfies
\begin{align}  
  \pseudoregret_T
  &= \sum_{t=1}^T \sum_{i \neq i^*} \E [v_{i,t}] \Delta_i \nonumber \\
  &\geq \Delta \sum_{t=1}^T \E \left[ \frac{\sum_{i \neq i^*} e^{-\eta_t (L_{i,t-1} - L_{i^*,t-1})}}{1 + \sum_{i \neq i^*} e^{-\eta_t (L_{i,t-1} - L_{i^*,t-1})}} \right] \nonumber \\
  &\geq \Delta \sum_{t=1}^T \E \left[ \frac{e^{-\eta_t (L_{j,t-1} - L_{i^*,t-1})}}{1 + e^{-\eta_t (L_{j,t-1} - L_{i^*,t-1})}} \right] \label{eq:proof-no-bernstein-b1} \\
  &\geq \frac{\Delta}{3} \sum_{t=1}^T \E \left[ \bm 1 \left( e^{-\eta_t (L_{j,t-1} - L_{i^*,t-1})} \geq \frac{1}{2} \right) \right] \nonumber \\
  &= \frac{\Delta}{3} \sum_{t=1}^T \P \left( \eta_t (L_{j,t-1} - L_{i^*,t-1}) \leq \log 2 \right) \label{eq:proof-no-bernstein-b2}
\end{align}
where~\eqref{eq:proof-no-bernstein-b1} relies on the fact that the function $x \mapsto \frac{x}{1+x}$ is increasing on $\R^+$.
Denoting $a = (\log 2)/(c_0 \sqrt{\log M})$, we have for every $1 \leq t \leq 1 + \frac{a^2}{4 \Delta^2}$:
\begin{align}  
  \P \left( \eta_t (L_{j,t-1} - L_{i^*,t-1}) > \log 2 \right)
  &= \P \left( L_{j,t-1} - L_{i^*,t-1} - \Delta (t-1) > a \sqrt{t} - \Delta (t-1) \right) \nonumber \\
  &\leq \P \left( L_{j,t-1} - L_{i^*,t-1} - \Delta (t-1) > \frac{a \sqrt{t-1}}{2} \right) \label{eq:proof-no-bernstein-aDelta} \\
  &\leq e^{- a^2/8} \label{eq:proof-no-bernstein-hoeffding}
\end{align}
where inequality~\eqref{eq:proof-no-bernstein-aDelta} stems from the fact that $\Delta (t-1) \leq \frac{a \sqrt{t-1}}{2}$ (since $t \leq 1 + \frac{a^2}{4 \Delta^2}$), while~\eqref{eq:proof-no-bernstein-hoeffding} is a consequence of Hoeffding's bound applied to the \iid $[-1-\Delta, 1-\Delta]$-valued random variables $\ell_{j,s} - \ell_{i^*,s} - \Delta$, $1\leq s \leq t-1$.
Assuming that $c_0 \geq 1$, we have $a \leq \sqrt{\log 2} \leq 1$, so that by concavity of the function $x \mapsto 1- e^{-x/8}$, $1-e^{-a^2/8} \geq (1-e^{-1/8})a^2$.
Combining this with inequalities~\eqref{eq:proof-no-bernstein-b2} and~\eqref{eq:proof-no-bernstein-hoeffding} and using the fact that $\left\lfloor 1+\frac{a^2}{4 \Delta^2} \right\rfloor \geq \frac{a^2}{4 \Delta^2}$, we obtain for $T \geq \frac{1}{4 \Delta^2} \geq \frac{a^2}{4 \Delta^2}$:
\begin{align}
  \label{eq:proof-no-bernstein-b3}
  \E \left[ R_T \right]
  \geq \frac{\Delta}{3} \min \left( \frac{a^2}{4 \Delta^2}, T \right) (1-e^{-1/8}) a^2 
  = \frac{(1-e^{-1/8}) a^4}{12 \Delta} 
  \geq \frac{1}{450 c_0^4 (\log M)^2 \Delta}
    \, ,
\end{align}
where the last inequality comes from the fact that $(\log 2)^4 (1-e^{-1/8})/12 \geq \frac{1}{450}$.


\end{document}